\documentclass[final]{IEEEtran}

\usepackage{amsmath,amssymb,amsthm,amsfonts}
\usepackage{graphicx}
\usepackage{algorithm,algorithmic}
\usepackage{cite}
\usepackage{tikz}
\usetikzlibrary{shapes,arrows,external}
\usetikzlibrary{positioning}
\usepackage{breqn}
\usepackage{multirow}
\usepackage{enumerate}
\usepackage{bbm}
\usepackage{nicefrac}  
\usepackage{subcaption}

\newtheorem{thm}{Theorem}

\newtheorem{lemma}{Lemma}

\newtheorem{defn}{Definition}

\newcommand{\bE}{\mathbb{E}}

\newcommand{\bR}{\mathbb{R}}

\newcommand{\bN}{\mathbb{N}}

\newcommand{\sF}{{\mathcal F}}

\newcommand{\sN}{{\mathcal N}}

\newcommand{\argmin}{\operatornamewithlimits{arg\ min}}

\newcommand{\ind}[1]{\mathbbm{1}\set{#1}}

\newcommand{\set}[1]{\left\{#1\right\}}
\newcommand{\abs}[1]{\left\lvert #1 \right\rvert}

\newcommand{\tnhat}{\hat{\theta}_{n}}
\graphicspath{{./Figures/}}

\begin{document}

\title{A Finite-Horizon Approach to Active Level Set Estimation}

\author{Phillip~Kearns, 
    Bruno~Jedynak,~\IEEEmembership{Senior Member,~IEEE},
    and~John~Lipor,~\IEEEmembership{Member,~IEEE}%
    \thanks{P. Kearns is with the Department of Chemistry and Biochemistry, University of Oregon, Eugene, OR 97403 USA (email: phillipk@uoregon.edu).}
    \thanks{B. Jedynak is with the Department of Mathematics and Statistics, Portland State University, Portland, OR 97201 USA (email: bruno.jedynak@pdx.edu).}
    \thanks{J. Lipor is with the Department of Electrical and Computer Engineering, Portland State University, Portland, OR 97201 USA (email:lipor@pdx.edu).}
}

\maketitle

\begin{abstract}
    We consider the problem of active learning in the context of spatial sampling for level set estimation (LSE), where the goal is to localize all regions where a function of interest lies above/below a given threshold as quickly as possible. We present a finite-horizon search procedure to perform LSE in one dimension while optimally balancing both the final estimation error and the distance traveled for a fixed number of samples. A tuning parameter is used to trade off between the estimation accuracy and distance traveled. We show that the resulting optimization
    problem can be solved in closed form and that the resulting policy generalizes existing approaches to this problem. We then show how this approach can be used to perform level set estimation in higher dimensions under the popular Gaussian process model. Empirical results on synthetic data indicate that as the cost of travel increases, our method's ability to treat distance nonmyopically allows it to significantly improve on the state of the art. On real air quality data, our approach
    achieves roughly one fifth
    the estimation error at less than half the cost of competing algorithms.
\end{abstract}

\begin{IEEEkeywords}
    Adaptive sampling, autonomous systems, dynamic programming, Gaussian processes, level set estimation, mobile sensors.
\end{IEEEkeywords}

\section{Introduction}
\label{sec:introduction}

In recent years, there has been a growing interest in autonomously sampling environmental phenomena, owing in part to the increasing occurrence of extreme events such as wildfires in the United States. In particular, the problem of adaptively sampling the environment to determine all regions where a phenomenon of interest is above or below a critical threshold---a problem known as \emph{level set estimation} (LSE)---has received a great deal of attention within the signal processing and
machine learning communities \cite{gotovos2013active,bogunovic2016truncated,lejeune2020thresholding}. The resulting algorithms are often designed with the goal of deployment on an autonomous, mobile sampling vessel, such as an unmanned aerial vehicle (UAV). Since these vehicles are tasked with covering regions on the order of hundreds of square kilometers, a key component of adaptive sampling methods is the ability to account for the costs associated with both the number of measurements taken and the distance traveled throughout the
sampling procedure.


As a motivating problem, we consider the task of tracking wildfires, where our goal is to determine the spatial extent of particulate matter 2.5 (PM 2.5) caused by wildfire smoke (see Fig.~\ref{fig:fireMap}).
Algorithms designed to rapidly determine the boundary of such a region fall within the category of \emph{active learning} or \emph{adaptive sampling} \cite{settles2012active,castro2008active} and typically try to maximize a notion of information gain per sample. However, this approach fails to account for the distance traveled throughout the sampling procedure. Hence, standard approaches to active learning based in search space reduction \cite{nowak2008generalized,dasgupta2005analysis,willett2006faster}
or adaptive submodularity \cite{golovin2011adaptive}, which seek to minimize only the number of samples taken, will be accompanied by potentially dramatic drawbacks in terms of total sampling cost. While the approaches in \cite{guillory2009average,bogunovic2016truncated} account for arbitrary costs, these treat costs myopically, failing to account for the expected future cost after sampling a given location. Newer, bisection-style search methods such as quantile search (QS)
\cite{lipor2017distance} and its extension \cite{lipor2018quantile} both achieve an explicit tradeoff between the number of samples and distance traveled. Although these improve upon previous methods in terms of total sampling time, neither guarantees to find the optimal search procedure.

\begin{figure*}[t]
    \centering
    \begin{subfigure}[t]{0.49\linewidth}
        \centering
        \includegraphics[width=0.96\linewidth]{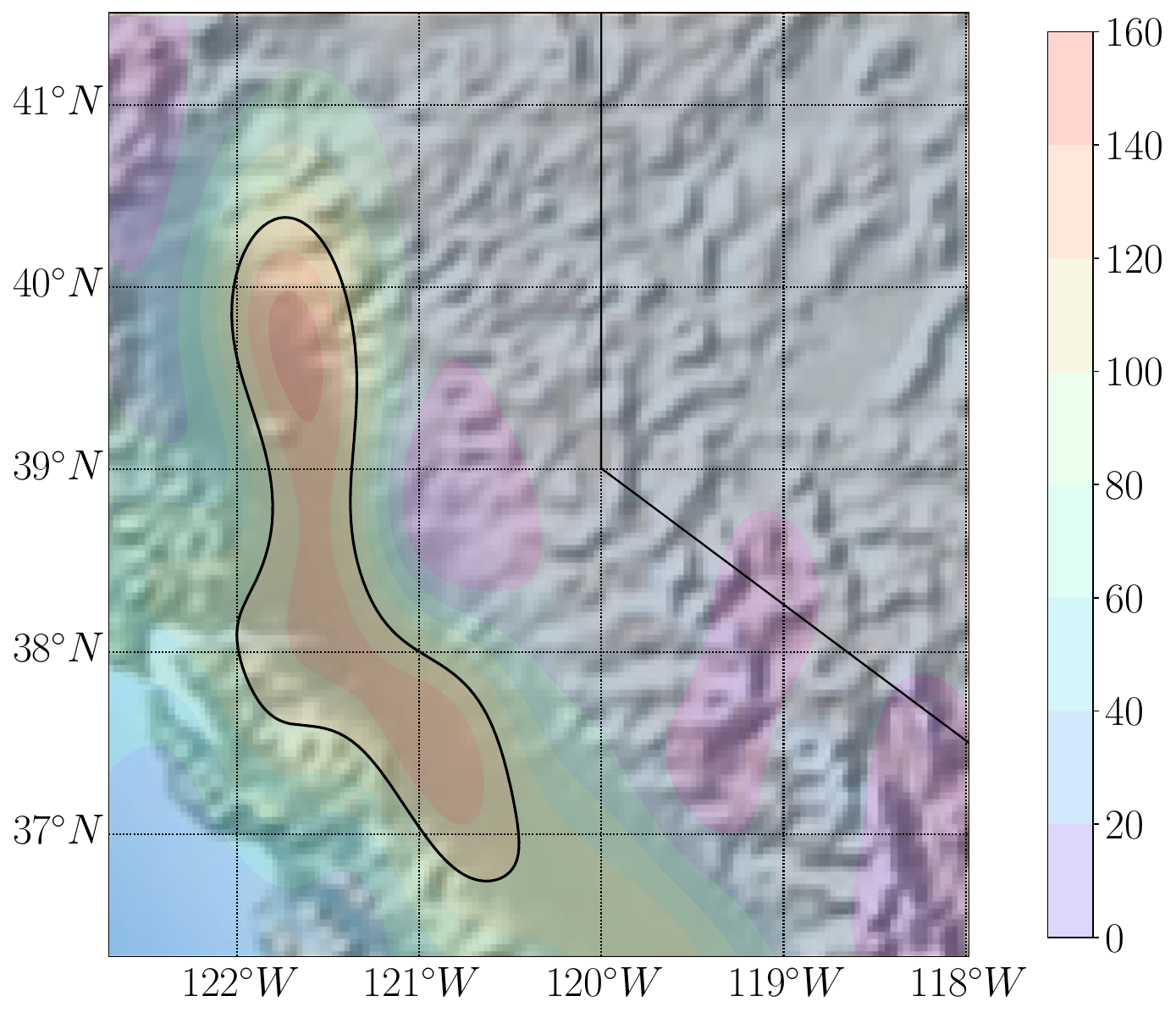}
        \caption{}
    \end{subfigure}
    \begin{subfigure}[t]{0.49\textwidth}
        \centering
        \includegraphics[width=\textwidth]{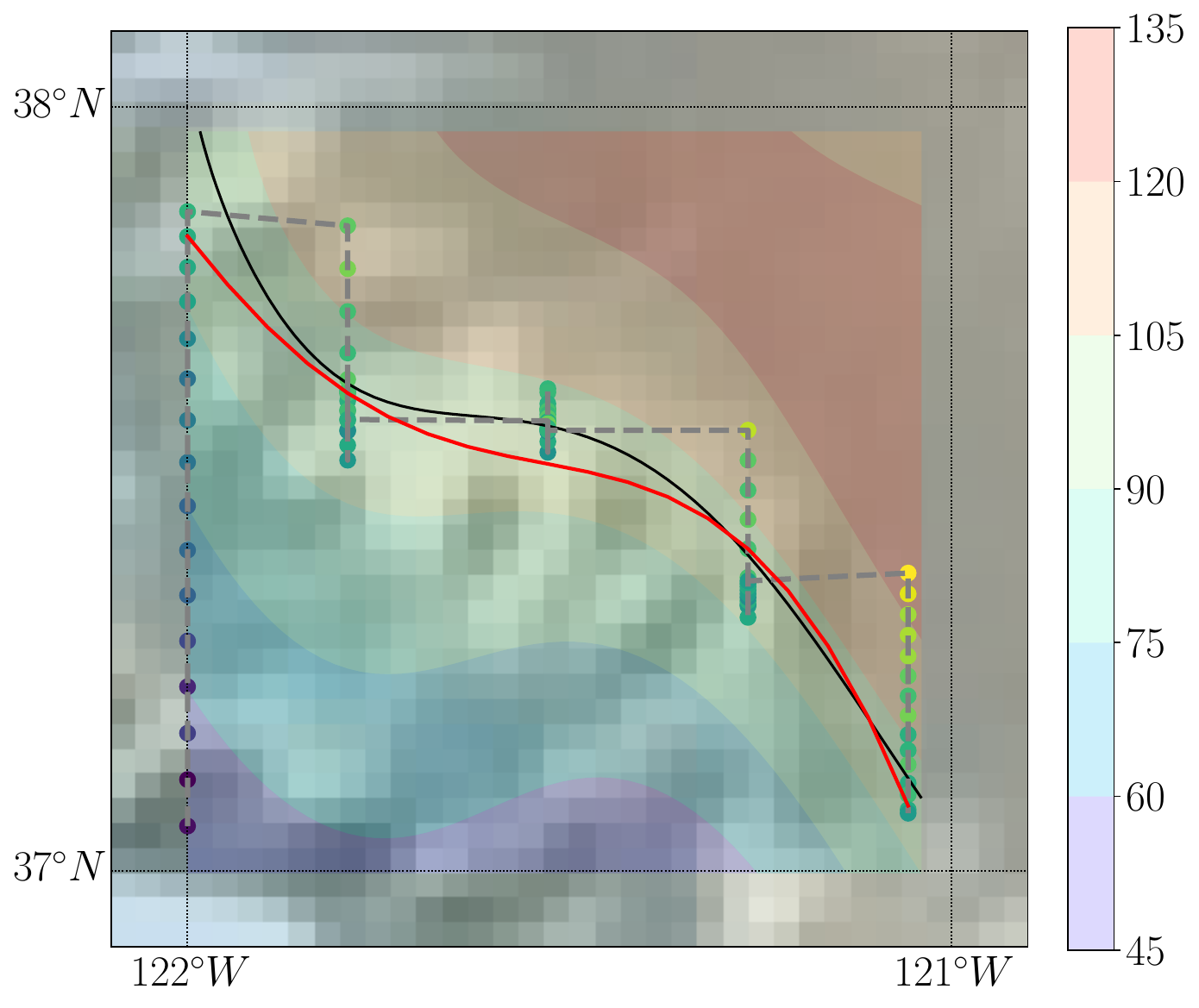}
        \caption{}
    \end{subfigure}
    \caption{Map of PM 2.5 following the California Camp Fire on November 18, 2018. (a) Full level set boundary (black) denoting all locations where PM 2.5 is above 100 $\mu$g / m$^{3}$. (b) Subset of region and samples collected via proposed method, which performs a series of one-dimensional searches. Dots denote sample locations, red solid line is estimated boundary, and gray dashed line is path traversed by sensor. Corresponding sub-region is approximately 111 km $\times$ 111 km.}
    \label{fig:fireMap}
\end{figure*}

In this paper, we present a finite-horizon sampling procedure that optimally balances the distance traveled with the final estimation error after obtaining $N$ measurements. At its extremes, this algorithm minimizes either the final entropy or total distance traveled, with a tradeoff achieved by varying a user-specified tuning parameter. 
This paper is an extension of the conference version \cite{kearns2019optimal}; our contributions beyond \cite{kearns2019optimal} are as follows. We prove that the policy obtained for the case of noiseless measurements is indeed a global minimum, as opposed to a critical point only. We then present a method for handling noisy measurements and prove this approach converges almost surely to the true change point of a one-dimensional step function. We extend this idea to show how the proposed search
algorithm can be used to perform LSE in Gaussian processes (GPs), a topic that has received a great deal of attention in the machine learning literature \cite{gotovos2013active,hitz2014fully,bogunovic2016truncated}. We provide extensive simulations on both synthetic data, as well as air quality data obtained from the AirNow database \cite{epa2020air}. Finally, we compare our approach with the state-of-the-art in Gaussian process level set estimation (GP-LSE) \cite{bogunovic2016truncated} and
demonstrate that the proposed method is capable of
estimating the level set at a lower sampling cost while requiring a fraction of the computation time.

\section{Problem Formulation \& Related Work}
\label{sec:problem}

As stated in the introduction, we are ultimately concerned with the problem of level set estimation in spatial domains, i.e., of estimating the superlevel set
\begin{equation}
    \label{eq:superlevelSet}
    S = \set{x \in \bR^{d}: f(x) \geq \gamma},
\end{equation}
where $f: [0,1]^{d} \to \bR$ is a function governing some phenomenon of interest and $\gamma$ is a user-defined threshold.
In this work, we are primarily concerned with the case where $f$ is a one-dimensional step function belonging to the class
\begin{equation*}
    \sF = \{f_{\theta} : f_{\theta}(x) = \ind{[0,\theta)}, \theta\in [0,1]\},
\end{equation*}
where $\ind{E}$ denotes the indicator function, which takes the value one when $x \in E$ and zero otherwise. In this case, the superlevel set is $S = \set{x \in [0,1]: x < \theta}$, and LSE is equivalent to estimating the change point $\theta$.
Although this model may seem highly restrictive, we will show that higher-dimensional level set boundaries can be estimated using a series of one-dimensional step functions, and that such an approach outperforms state-of-the-art algorithms that consider the two-dimensional problem directly. An example of using a series of one-dimensional searches to estimate a two-dimensional boundary is illustrated in Fig.~\ref{fig:fireMap}(b).

To perform boundary estimation, our sampling proceeds as follows. Assume we obtain observations 
$\{Y_{n}\}_{n = 1}^{N} \in \set{0,1}^{N}$ from the sample locations $\{X_{n}\}_{n = 1}^{N}$ in the unit interval in a sequential fashion according to $Y_{n} = \ind{x \in S}$, where $S$ is the superlevel set defined in \eqref{eq:superlevelSet}. In the case of one-dimensional step functions, each sample obtained reduces the interval in which the change point may lie. Our goal is then to estimate the change point location while minimizing the sampling cost for a fixed number of samples, a function of both the final expected interval size \textit{and} expected distance traveled. 

\subsection{Related Work}
\label{sec:related}

A variety of adaptive sampling methods have been proposed throughout the literature from various communities.
Several popular approaches to LSE assume that the underlying function of interest is a GP \cite{williams2006gaussian}, which yields a posterior distribution on the value at each point. In \cite{gotovos2013active}, the authors leverage the GP model to construct confidence intervals around the value of each point, sequentially sampling points of highest ambiguity. This approach was extended in \cite{hitz2014fully} and \cite{bogunovic2016truncated}, providing novel approaches for
sampling while also accounting for the distance traveled between points. Another method for path-efficient LSE that seeks to reduce the distance traveled by the mobile sensor is proposed in \cite{bottarelli2017path}, but this method assumes the vehicle can continuously acquire measurements with a negligible cost.

The authors of \cite{singh2006active} introduce the idea of adaptive data collection for mobile path planning, or \textit{informative path planning}, where previous samples are used to guide the motion of the sensing vehicles for further sampling. Algorithms for informative path planning typically focus on maximizing information gain over a scalar field for an underwater autonomous vehicle. The approaches presented in
\cite{singh2006active,krause2010auv,zhang2007adaptive} accommodate a wide range of sampling scenarios that include varied sampling time, path constraints, and limited battery. However, these methods often require a coarse sampling of the entire region of interest, which is not feasible for the large spatial regions considered here. In contrast, boundary detection methods like those in
\cite{marthaler2004tracking,cannell2005comparison,jin2007environmental} use mobile sensors to map a spatial threshold as closely as possible. These methods provide efficient and accurate mappings of a binary classification boundary but unfortunately do not account for the cost of obtaining each sample. 

Among the approaches from active learning, many rely on the principle of search space reduction \cite{dasgupta2005analysis,willett2006faster,nowak2008generalized}, which aims to rapidly reduce the set of points where the level set boundary may lie through intelligent sampling. In general, these methods do not permit the inclusion of distance-based or other penalties, and as a result they tend to yield bisection-type solutions \cite{burnashev1974interval} that require few samples but may
travel large distances. 
These methods can also be viewed as ``greedy'' approaches that aim to maximize search space reduction at each step. While greedy methods have been shown to be near optimal in terms of sample complexity \cite{nowak2008generalized,dasgupta2005analysis}, they frequently ignore additional costs that may be incurred during the sampling procedure. Moreover, methods such as \cite{donmez2008proactive} that greedily incorporate realistic costs into the algorithm formulation have been shown to perform worse than the alternative approaches when applied to distance-penalized searches \cite{lipor2017distance}. 

A popular greedy approach to active learning relies on the concept of adaptive submodularity (AS) \cite{golovin2010adaptive}. AS is a diminishing returns principle, which informally states that samples are more valuable early in the search procedure. The work of \cite{golovin2011adaptive} shows that a greedy procedure is optimal up to a constant factor for several active learning problems, including the case of nonuniform label costs. However, AS is a property of set functions, and does not consider a sequential dependency among sampling locations.  
A notion of submodular optimization with sequential dependencies was presented in \cite{tschiatschek2017selecting}, but the proposed algorithm relies on a reordering procedure that is not applicable to our problem.
While \cite{guillory2009average} provides a theoretical analysis of greedy active learning with non-uniform costs, the authors only consider the case of query costs being fixed. In contrast, our scenario has non-uniform and dynamic costs, where travel time depends on the distance between points.

Of primary relevance to the work presented in this paper is the work of \cite{lipor2017distance}, which introduces the quantile search (QS) algorithm for determining the change point of a one-dimensional step function while balancing the above costs. QS is a generalization of binary bisection \cite{castro2008minimax,horstein1963sequential,burnashev1974interval}, where the idea is that by successively sampling a fixed fraction $1/m$, where $m > 2$, into the remaining feasible interval, the
desired tradeoff between number of samples and distance traveled can be achieved. The authors characterize the expected error after a fixed number of samples as well as the distance traveled; they further provide an approach to handling noisy measurements and prove its convergence. This work was extended in \cite{lipor2018quantile}, introducing the uniform-to-binary (UTB) algorithm, where the key observation is that QS can be improved by allowing the search parameter $m$ to
vary with time. While both QS and UTB provide promising empirical results, neither algorithm provides guarantees of optimality in terms of the total sampling cost. Most recently, an optimal approach based on dynamic programming was presented in \cite{wang2019distance}, where the search procedure is cast as a stochastic shortest path problem \cite[Ch. 2]{bertsekas2007dynamic}. However, the use of dynamic programming requires an increased computation time, and the resulting
solution depends heavily on the discretization used. In this work, we present an approach that strictly generalizes the QS algorithm while still admitting a closed-form solution that can be easily deployed on a mobile sensing device. 

\section{Finite Horizon Search}
\label{sec:algorithm}

In this section, we describe our approach to distance-penalized LSE, which we refer to as \emph{finite horizon search} (FHS). To appropriately penalize distance, FHS considers a fixed number of measurements (i.e., a finite sampling horizon) and optimizes the weighted sum of distance traveled and entropy in the posterior distribution of the change point $\theta$ after obtaining these measurements. A tuning parameter allows the user to control the importance of distance penalization,
resulting in binary search at one extreme and sampling adjacent locations at the other.
In the case of noiseless measurements, we show that the optimal sampling policy can be obtained in closed form. We then show how this policy can be extended to handle noisy observations and prove the resulting method converges almost surely to the true change point. Finally, we describe an approach to the well-studied GP-LSE problem in the case
where the level set
boundary can be written as a function in one dimension.


\subsection{Noiseless Measurements}
\label{sec:noiseless}

We first consider the simple case of noiseless, binary-valued measurements, where our goal is to estimate the change point $\theta$ on the unit interval.
It is convenient, while not restrictive, to define search strategies in terms of the \emph{fraction of the remaining interval} to move at each step, whether forward or backward, in an analogous fashion to \cite{lipor2017distance,lipor2018quantile}. The resulting class of policies is adaptive to the unknown location of $\theta$ and non-restrictive in the sense that any optimal policy will not sample in locations with probability zero (locations outside the remaining interval).

Begin with a uniform prior on the change point $\theta$, and let the $N$ fractions be $\set{z_{n}}_{n=1}^{N}$, where $z_{n} \in [0,1]$ for $n = 1,\dots,N$. A straightforward Bayesian update yields the posterior distribution after each sample. Let $H_N$ be the entropy of the posterior distribution after $N$ observations, $D_N$ be the total distance traveled, and $\lambda \geq 0$ be a tuning parameter that governs the tradeoff between these costs. We define the expected sampling cost after $N$
observations as the weighed sum of exponentiated entropy and distance
\begin{equation}
    J(z_{1},\dots,z_{N}) = \bE_{\theta}\left[ e^{H_{N}} + \lambda D_N \right].
    \label{eq:totCost}
\end{equation}
Note that for a uniform distribution on an interval of length $a$, $e^{H_{N}} = e^{\log(a)} = a$; thus, when beginning with a uniform prior on $\theta$, \eqref{eq:totCost} is equivalent to minimizing a weighted combination of the (expected) final interval length and distance traveled. In what follows, we will derive a closed-form solution to this problem, as well as a means of computing the number of samples required to obtain an expected interval length below a given threshold.

\subsubsection{Closed-Form Solution}
\label{sec:analysis}

We now demonstrate that the global optimum of \eqref{eq:totCost} can be found in closed form. We first define the \emph{feasible interval} as the interval in which the change point may lie. Formally,
\begin{defn}
    Assume $n$ measurements at locations $X_{1},\dots,X_{n}$ have been obtained and define
    \begin{equation*}
        X_{l} = \max \set{X_{i} \in X_{1},\dots,X_{n}: Y_{i} = 1}
    \end{equation*}
    and
    \begin{equation*}
        X_{u} = \min \set{X_{i} \in X_{1},\dots,X_{n}: Y_{i} = 0}.
    \end{equation*}
    Then the \emph{feasible interval} after $n$ samples is $[X_{l}, X_{u}]$.
\end{defn}
To derive the optimal sampling fractions, we begin by rewriting \eqref{eq:totCost} in terms of the expected size of the feasible interval at each step, recognizing that the distance traveled at a given step is equal to the product of the interval size and the sampling fraction. The resulting cost function can be differentiated to find a critical point. The principle of dynamic programming verifies that the resulting solution is indeed a global optimum.

\begin{thm}
    \label{thm:closedForm}
    Let $\lambda \in [0,2)$ and assume the unknown change point has distribution $\theta \sim \text{Unif}([0,1])$. Further, assume the $N$ measurements are defined via $N$ fractions $z_{1},\dots,z_{N} \in [0,1]$ denoting the proportion of the current feasible interval to sample. Define $\xi_{i} = z_{i}^{2} + (1 - z_{i})^{2}$. Then the critical points of the cost function \eqref{eq:totCost} are of the form 
    \begin{equation}
        z^*_{k} = \frac{1}{2} - \lambda \frac{1}{4\rho_k}, \quad k = 1,\dots,N,
        \label{eq:optFractions}
    \end{equation}
    where $\rho_{N} = 1$ and
    \begin{align}
        \label{eq:optRho}
        \rho_k = \prod_{i=k+1}^{N}\xi_i + \lambda \sum_{i=k+1}^{N} z_i^{*} \prod_{j=k+1}^{i-1}\xi_j, \quad k = 1,\dots,N-1,
    \end{align}
    depends only on the fractions $z_{k+1},\dots,z_{N}$. 
\end{thm}
\begin{proof}
    The proof begins by rewriting the cost function in terms of the expected length of the feasible interval. Let $H_{N}$ be the length of the feasible interval after $N$ measurements. Then by Lemma~\ref{lem:expLength} (see Appendix), we have that 
    \begin{equation*}
        \bE \left[ e^{H_{N}} \right] = \prod_{i = 1}^{N} \left( z_{i}^{2} + (1 - z_{i})^{2} \right) = \prod_{i=1}^{N} \xi_{i}.
    \end{equation*}
    Let $D_{N}$ be the distance traveled after $N$ samples. Note that this distance is exactly the product of the interval length and the fraction of the interval to be traveled at each step, i.e.,
    \begin{equation*}
        D_{N} = \sum_{i = 1}^{N} z_{i} e^{H_{i-1}}.
    \end{equation*}
    Therefore
    \begin{align*}
        \bE\left[ D_{N} \right] & = \bE\left[ \sum_{i = 1}^{N} z_{i} e^{H_{i-1}} \right] \\
        & = \sum_{i = 1}^{N} z_{i} \bE\left[ e^{H_{i-1}} \right].
    \end{align*}
    Applying Lemma~\ref{lem:expLength} then yields
    \begin{equation*}
        \bE \left[ D_{N} \right] = \sum_{i = 1}^{N} z_{i} \prod_{j = 0}^{i-1} \xi_j.
    \end{equation*}
    We can therefore rewrite the cost function \eqref{eq:totCost} as
    \begin{eqnarray}
        J\left( z_{1},\dots,z_{N} \right) &=& \bE \left[ e^{H_{N}} \right] + \lambda \bE \left[ D_{N} \right] \nonumber \\ 
        &=& \prod_{i=1}^{N}\xi_i + \lambda \sum_{i=1}^{N} z_i \prod_{j=0}^{i-1}\xi_j.
        \label{eq:totCost2}
    \end{eqnarray}
    After rewriting in the form \eqref{eq:totCost2}, we can easily compute the gradient to be
    \begin{equation}
        \label{eq:costGradient}
        \frac{\partial J}{\partial z_l} = \left(\prod_{i=1}^{l-1} \xi_i\right) \left( \left(4z_l - 2\right) \rho_l + \lambda \right),
    \end{equation}
    and setting the gradient to zero yields
    \begin{equation*}
        z_l = \frac{1}{2} - \lambda \frac{1}{4\rho_l}.
    \end{equation*}
\end{proof}

Thm.~\ref{thm:closedForm} characterizes the critical points of \eqref{eq:totCost}. Although setting \eqref{eq:costGradient} to zero yields a unique solution, this is not sufficient to guarantee global optimality (a global optimum could lie on the boundary of $[0,1]^{N}$). Further, even in its simplified form \eqref{eq:totCost2}, the cost function is a high-order polynomial whose convexity is difficult to analyze. 

\begin{thm}
    The critical point characterized by Thm.~\ref{thm:closedForm} is the global optimum of the cost function \eqref{eq:totCost}.
\end{thm}
\begin{proof}
    The argument is based on the dynamic programming lemma \cite{bertsekas2005dynamic}, restated here for convenience. 
    \begin{lemma}
        \label{lem:dp}
        Suppose a sequence $z_{1}^{*},\dots,z_{n}^{*}$ is such that for any $z_{1},\dots,z_{n}$
        \begin{equation}
            \label{eq:dp1}
            J(z_{1},\dots,z_{n}) \geq J(z_{1},..,z_{n-1},z_{n}^{*}) 
        \end{equation}
        and for $1 \leq p \leq n-1$ and any $z_{1},\dots,z_{p}$ that
        \begin{equation}
            \label{eq:dp2}
            J(z_{1},\dots,z_{p},z_{p+1}^{*},\dots,z_{n}^{*}) \geq J(z_{1},\dots,z_{p}^{*},z_{p+1}^{*},\dots,z_{n}^{*}).
        \end{equation}
        Then for all sequences $z_{1},\dots,z_{n}$
        \begin{equation}
            J(z_{1},\dots,z_{n}) \geq J(z_{1}^{*},\dots,z_{n}^{*}).
        \end{equation}
    \end{lemma}


    We verify that the local minimum defined in Thm.~\ref{thm:closedForm} satisfies the hypothesis of Lemma~\ref{lem:dp}. 
    For any $z_{1},\dots,z_{p}$, define $\pi_{p} = \prod_{i=1}^{p-1} \xi_{i}$. To verify the statement \eqref{eq:dp1}, for a fixed $z_{1},\dots,z_{n-1}$, we let
    \begin{eqnarray*}
        f_{n}(z) &=& J(z_{1},\dots,z_{n-1},z) \\
        &=& \pi_{n-1} \xi + \lambda z \pi_{n-1} + \lambda \sum_{i=1}^{n-1} z_{i} \pi_{i-1},
    \end{eqnarray*}
    where $\xi = z^{2} + (1-z)^{2}$. The final term above does not depend on $z$, indicating that $f_{n}(z)$ is a second-order polynomial in $z$. Moreover, $\pi_{n-1} > 0$, so the above is strictly convex, and hence a unique global minimizer can be found by differentiation. It is easily verified that this corresponds to the critical point found in Thm.~\ref{thm:closedForm}. We next verify statement \eqref{eq:dp2}. Let
    \begin{align*}
        \begin{split}
        f_{p}(z) &= J(z_{1},\dots,z_{p-1},z,z_{p+1}^{*},\dots,z_{n}^{*}) \\
        &= \pi_{p-1} \left( \xi \prod_{i=p+1}^{n} \xi_{i}^{*} + \right. \\
        &\qquad \left. \lambda\left( z + z_{p+1}^{*} \xi + \dots + z_{n}^{*} \xi \prod_{i=p+1}^{n-1} \xi_{i}^{*} \right) \right) \\
        &= \pi_{p-1} \left( \lambda z + \xi \rho_{p}^{*} \right) + \lambda \sum_{i=1}^{p-1} z_{i} \pi_{i-1}.
        \end{split}
    \end{align*}
    The above is again a second-order polynomial in $z$ and therefore convex. Minimization through differentiation again yields a global minimizer that corresponds with the critical point defined in Thm.~\ref{thm:closedForm}. Therefore, both statements of Lemma 1 are satisfied for the sequence $z_{1}^{*},\dots,z_{N}^{*}$ defined in Thm.~\ref{thm:closedForm}, indicating that the critical points are indeed global minimizers of the cost function \eqref{eq:totCost}.
\end{proof}

The above results show that the optimal $N$-step FHS policy can be obtained in closed form. Further, the sampling fractions can be computed in linear time, beginning with $z_{N} = 1/2 - \lambda/4$ and proceeding backwards.
While the above considers the case of the unit interval, it is straightforward to show that the cost \eqref{eq:totCost} is linear in the length of the interval, and hence the search fractions are independent of the initial length.
As a first observation, we note that when $N = 1$, FHS is a greedy approach that minimizes the one-step lookahead for the value function without concern for future consequences. In this case, the policy samples a constant fraction into the feasible interval, independent of the size of this interval. This is exactly the QS approach described in \cite{lipor2017distance}, and thus QS may be considered an instance of our proposed method with $N = 1$. 

Examining Thm.~\ref{thm:closedForm}, we see that $\rho_k$ is monotonically increasing in $k$, and hence the sampling fractions are monotonically increasing with $k$, as can be seen in Fig.~\ref{fig:OptFracs}. The resulting behavior is to perform small movements early on in the sampling procedure, when the feasible interval is large, avoiding large movements at the sacrifice of information gain. As the feasible interval is reduced, the steps get proportionally larger and emphasize
information gain/entropy reduction, since the incurred distance penalty is smaller. This extends the intuition behind the UTB sampling procedure of \cite{lipor2018quantile} in a more principled manner.

\begin{figure}[t]
    \centering
    \includegraphics[width=3.5in]{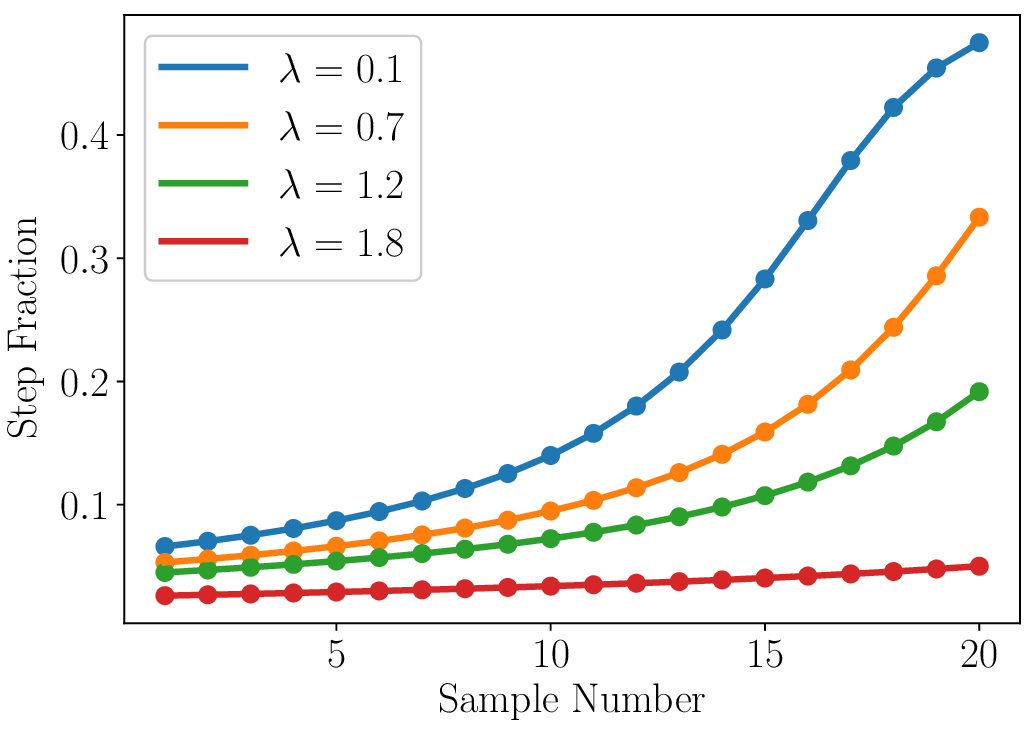}
        \caption{Sampling behavior of proposed FHS policy, showing fractions of the interval to travel at each step of a 20-step policy, where a larger $\lambda$ penalizes distance traveled more heavily. The FHS policy increases the sampling fraction over time to avoid traveling large distances.}
    \label{fig:OptFracs}
\end{figure}

The closed-form policy also provides insight into the range of admissible values of $\lambda$. Taking $\lambda = 0$ results in all fractions taking the value of 1/2, which is consistent with the well-known fact that binary bisection performs entropy minimization \cite{waeber2013bisection}. A higher value for the distance penalty parameter $\lambda$ results in a less aggressive policy, as the higher cost for potential overshoot encourages smaller steps. When $\lambda \geq 2$, the cost of traveling to obtain a
measurement, $\lambda z_1$, is larger than the expected reduction in entropy, $1-\xi_1$, and the trivial sample which requires no displacement is preferred. This is seen most directly by the fact that the final step is $z_{N} = 0$ in this case, making $\rho_{k} = 1$ for all $k$, and therefore all sampling fractions identically zero.

Finally, we note that \eqref{eq:totCost} may be solved directly using dynamic programming by discretizing the interval and allowing states to correspond to the possible lengths of the feasible interval. The corresponding cost is then the $\lambda$-penalized distance traveled at each step, with a terminal cost corresponding to the length of the final interval. In this light, the closed-form policy above may be viewed as an instance of dynamic programming, where each subproblem is 
solved in closed form.


\subsubsection{Searching with a Fixed Estimation Error}
\label{sec:searchProcedure}

In certain instances, a user may wish to terminate the search procedure when the length of the feasible interval is below a given threshold to guarantee a fixed estimation error. By noting that the exponential entropy is equal to the feasible interval length, the result of Lemma~\ref{lem:expLength} can be used to determine the number of steps required to reduce the expected interval length below the threshold $\varepsilon > 0$. In particular, we leverage the fact that the tail
subproblem of length $N-k$ is equivalent to the solution to the $(N-k)$-step problem for any $k \in \set{1,\dots,N-1}$, and hence the expected interval size can be computed sequentially until it is below the threshold $\varepsilon$.
Further, it is easy to show that intervals of arbitrary length $L$ are reduced by the same fraction, and hence we are not restricted to intervals of unit length.
Pseudocode for determining the expected number of samples and search fractions for each sample is given in Alg.~\ref{alg:Ent}. 

\begin{algorithm}[t]
    \caption{Policy calculation for fixed estimation error}
    \label{alg:Ent}
    \begin{algorithmic}[1]
        \STATE \textbf{Input:} stopping error $\varepsilon > 0$, distance penalty $\lambda \in [0,2)$, initial interval length $L$
        \STATE \textbf{Initialize:} $z_N \gets \frac{1}{2} - \frac{\lambda}{4}$, $\xi_{N} = \frac{1}{2} + \frac{\lambda^{2}}{8}$, $k \gets 0$ \\
        \WHILE{$L\prod_{i=N-k}^{N} \xi_i > \varepsilon$}
        \STATE $k \gets k+1$
        \STATE compute $\rho_{N-k}$ according to \eqref{eq:optRho}
        \STATE $z_{N-k} \gets \frac{1}{2} - \lambda/(4\rho_{N-k}) $
        \STATE $\xi_{N-k} \gets z_{N-k}^{2} + (1-z_{N-k})^{2}$
    \ENDWHILE
    \STATE $N \gets k$
    \end{algorithmic}
\end{algorithm}

In the case where a search terminates only after a given estimation error has been obtained, we follow a two-phase procedure. Before the search begins, we use the method presented in Alg.~\ref{alg:Ent} to calculate the $N$ steps such that the expected final interval size is less than $\varepsilon$. Then, in the first search stage, samples are taken according to this $N$-step policy. If the feasible interval is smaller than the desired threshold before all $N$ samples have been taken, the
search terminates. Otherwise, the algorithm performs a greedy search (optimal 1-step policy, line 7) until the interval is sufficiently small. Pseudocode for this procedure is provided in Alg.~\ref{alg:FH}.

\begin{algorithm}[t]
    \caption{Finite Horizon Search}
    \label{alg:FH}
    \begin{algorithmic}[1]
        \STATE \textbf{Input:} search fractions $z_{1}, \dots, z_{N}$, stopping error $\varepsilon$
        \STATE \textbf{Initialize:} $X_{0} \gets 0$, $Y_{0} \gets 1$, $a \gets 0$, $b \gets 1$, $n \gets 1$ 
        \WHILE{$b-a > \varepsilon$}
        \IF{$n \leq N$}
            \STATE{$z \gets z_n$}
        \ELSE 
            \STATE {$z \gets \frac{1}{2} - \frac{\lambda}{4} $}
        \ENDIF
        \IF{$Y_{n-1} = 1$}
            \STATE $X_{n} \gets X_{n-1} + z(b-a)$
        \ELSE
            \STATE $X_{n} \gets X_{n-1} - z(b-a)$
        \ENDIF
        \STATE $Y_{n} \gets f(X_{n})$
        \STATE $a = \max \set{X_{i}: Y_{i} = 1, i \leq n}$
        \STATE $b = \min \set{X_{i}: Y_{i} = 0, i \leq n}$
        \STATE $\tnhat \gets \frac{a + b}{2}$
        \STATE $n \gets n + 1$
        \ENDWHILE
    \end{algorithmic}
\end{algorithm}

\subsection{Noisy Measurements}
\label{sec:noisy}

In Section~\ref{sec:noiseless}, we assume the measurements are obtained in a noiseless manner, i.e., $Y_{i} = f_{\theta}(X_{i})$ exactly. However, low-cost environmental sensors are known to obtain measurements corrupted by noise. Further, in most real-world scenarios, the measurements themselves are real valued and then discretized to values of 0 (below level set threshold) or 1 (above threshold). In this case, values obtained near the true level set boundary are more likely to be erroneous, since small perturbations of the measurement can result in an
incorrect assignment. The work of \cite{burnashev1974interval,castro2008minimax,lipor2017distance} accounts for noisy binary-valued measurements by maintaining a posterior distribution over the change point $\theta$ and sampling at quantiles of this distribution. 
However, these assume both a constant search fraction and a constant probability of erroneous measurements (i.e., a bit flip with probability $p$).
In this section, we show how the policy derived from FHS can be extended to handle continuous-valued measurements corrupted by Gaussian noise and
prove that the resulting method converges almost surely to the true change point. 

\begin{algorithm}[ht]
    \caption{Probabilistic Finite Horizon Search}
    \label{alg:pfhs}
    \begin{algorithmic}[1]
        \STATE \textbf{Input:} search fractions $z_{1}, \dots, z_{N}$, noise variance $\sigma^{2}$, stopping error $\varepsilon$
        \STATE \textbf{Initialize:} $X_{0} = 0$, $\pi_{0}(x) = 1$ for all $x \in [0,1]$, $n \gets 1$ 
        \WHILE{$\bE_{\theta \sim \pi_{n}} \abs{\hat{\theta}_{n} - \theta} > \varepsilon$}
        \IF{$n \leq N$}
            \STATE{$z \gets z_n$}
        \ELSE 
            \STATE {$z \gets \frac{1}{2} - \frac{\lambda}{4} $}
        \ENDIF
        \STATE set $\tilde{X}_{0}, \tilde{X}_{1}$ such that
        \begin{equation*}
            \int_{0}^{\tilde{X}_{0}} \pi_{n-1}(x) = z \quad \text{and} \quad \int_{\tilde{X}_{1}}^{1} \pi_{n-1}(x) = 1 - z
        \end{equation*}
        \STATE $X_{n} \gets \argmin_{X \in \set{\tilde{X}_{0}, \tilde{X}_{1}}} \abs{X_{n-1} - X}$
        \STATE $Y_{n} \gets f(X_{n})$
        \IF{$Y_{n} > \gamma$}
        \STATE $p_{n} \gets 1 - \Phi\left( \frac{Y_{n} - \gamma}{\sigma} \right)$
        \STATE $\pi_{n}(x) = \begin{cases}
            \frac{p_{n}}{z \circ p_{n}} \pi_{n-1}(x), & x \leq X_{n} \\
            \frac{1 - p_{n}}{z \circ p_{n}} \pi_{n-1}(x), & x > X_{n}
        \end{cases}
        $
        \ELSE
        \STATE $p_{n} \gets \Phi\left( \frac{Y_{n} - \gamma}{\sigma} \right)$
        \STATE $\pi_{n}(x) = \begin{cases}
            \frac{1 - p_{n}}{z * p_{n}} \pi_{n-1}(x), & x \leq X_{n} \\
            \frac{p_{n}}{z * p_{n}} \pi_{n-1}(x), & x > X_{n}
        \end{cases}
        $
        \ENDIF
        \STATE $\hat{\theta}_{n} \gets \text{median}_{x} \pi_{n}(x)$
        \STATE $n \gets n + 1$
        \ENDWHILE
    \end{algorithmic}
\end{algorithm}

To handle noisy measurements, we utilize a probabilistic approach as in \cite{burnashev1974interval,lipor2017distance}, in which we sample a fraction into the posterior distribution on $\theta$ instead of the remaining interval. Beginning with a uniform prior over the change point, a posterior distribution $\pi_{n}(\theta)$ is obtained after each measurement via a Bayesian update. Consider sampling a fraction $z$ into the distribution $\pi_{n-1}$, resulting in the measurement location
$X_{n}$. In the case where $Y_{n} > \gamma$, the resulting update is
\begin{equation}
    \pi_{n}(x) = \begin{cases}
        \frac{p_{n}}{z \circ p_{n}} \pi_{n-1}(x), & x \leq X_{n} \\
        \frac{1 - p_{n}}{z \circ p_{n}} \pi_{n-1}(x), & x > X_{n},
    \end{cases}
    \label{eq:posUpdate}
\end{equation}
where $p_{n}$ is the probability of of an erroneous binary measurement and
\begin{equation*}
    z \circ p := zp + (1-z)(1-p).
\end{equation*}
In this case, a positive measurement indicates that the change point likely lies to the right of $X_{n}$, but there is still a nonzero probability that the change point is to the left, due to the possible erroneous measurement.
Similarly, for $Y_{n} < \gamma$, the Bayesian update becomes
\begin{equation}
    \pi_{n}(x) = \begin{cases}
        \frac{1 - p_{n}}{z * p_{n}} \pi_{n-1}(x), & x \leq X_{n} \\
        \frac{p_{n}}{z * p_{n}} \pi_{n-1}(x), & x > X_{n},
    \end{cases}
    \label{eq:negUpdate}
\end{equation}
where 
\begin{equation*}
    z * p := z(1-p) + (1-z)p.
\end{equation*}
After each update, the estimate $\hat{\theta}_{n}$ is taken to be the median of the resulting posterior distribution, and the expected absolute error is computed using the distribution $\pi_{n}$.

The above Bayesian updates assume binary-valued measurements and require the probability of an erroneous measurement. To handle more realistic sampling scenarios, we assume measurements are corrupted by zero-mean Gaussian noise, so that
\begin{equation}
    Y_{i} = f(X_{i}) + \zeta_{i} \sim \sN\left( f(X_{i}), \sigma^{2} \right),
    \label{eq:gaussianMeas}
\end{equation}
where $\sN(\mu, \sigma^{2})$ denotes the normal distribution with mean $\mu$ and variance $\sigma^{2}$. These measurements are then thresholded based on whether they are above or below the level set threshold $\gamma$.
Let $\Phi(\cdot)$ denote the cumulative distribution function of a standard normal random variable.
Under the measurement model \eqref{eq:gaussianMeas}, if we measure $Y_{i} < \gamma$ when $f(X_{i}) > \gamma$, an error occurs with probability $p_{i} = \Phi\left( \frac{Y_{i}-\gamma}{\sigma} \right)$. Similarly, if $Y_{i} > \gamma$ but $f(X_{i}) < \gamma$, an error occurs with probability $p_{i} = 1 - \Phi\left( \frac{Y_{i} - \gamma}{\sigma}
\right)$. 
Note that in both cases, the probability of error $p_{i}$ depends both on the noise variance $\sigma^{2}$ and the distance from the level set threshold. This is essential, as a measurement far from the level set threshold can handle much larger corruptions than one for which $\abs{Y_{i} - \gamma}$ is small.

Our search procedure computes this noise level after each measurement, updating the posterior to reflect high uncertainty when samples are obtained near the change point. Note that for a given search fraction $z_{n}$, equal information is gained by moving to the $z_{n}$ quantile of $\pi_{n}$ or the $1 - z_{n}$ quantile. To account for the goal of minimizing the distance traveled, we move to the nearer of these two quantiles at each measurement. Finally, to ensure the algorithm always moves
toward the median of the posterior, we follow the truncation approach of \cite{lipor2017distance}. We refer to this algorithm as \textit{probabilistic finite horizon search} (PFHS), and pseudocode is given in Alg.~\ref{alg:pfhs}. In the case where $\sigma = 0$, PFHS is exactly equivalent to FHS.

Below we show that, given sufficiently many measurements, a discretized version of the PFHS algorithm converges almost surely to the true change point. This approach discretizes the unit interval into bins of width $\Delta$ and facilitates analysis more easily than the continuous version \cite{castro2008active,waeber2013bisection}. We extend the analysis laid out in \cite{lipor2017distance} to allow for varying noise level and varying step sizes. 
\begin{thm}
    Assume measurements are obtained following the noise model \eqref{eq:gaussianMeas}. Then a discretized version of the PFHS algorithm converges almost surely to the true change point.
\end{thm}
\begin{proof}
    We wish to show that for any $\varepsilon > 0$ and any $\theta \in [0,1]$
    \begin{equation}
        \Pr\left( \limsup_{n \to \infty} \sup_{\theta \in [0,1]} \abs{\hat{\theta}_{n} - \theta} > \varepsilon \right) = 0.
        \label{eq:asCond}
    \end{equation}
    For any $\varepsilon > 0$, set $\Delta < \varepsilon$ such that $\Delta^{-1} \in \bN$ and consider a discretized probabilistic search algorithm with bin size $\Delta$. By \cite{lipor2017distance}, for any $\Delta > 0$, the discretized probabilistic search algorithm that samples a fraction $z$ into the posterior with a noise level $p < 1/2$ satisfies
    \begin{equation*}
        \begin{split}
            \sup_{\theta \in [0,1]} \Pr \left(|\tnhat - \theta| > \Delta \right) \leq \frac{1 - \Delta}{\Delta} t(z)^{n},
        \end{split}
    \end{equation*}
    where $\alpha = \sqrt{p}/(\sqrt{p} + \sqrt{1-p})$ and
    \begin{equation}
        t(z) := \frac{1-p}{2(1-\alpha)} \frac{p}{2\alpha} + \left( \frac{1-p}{2(1-\alpha)} - \frac{p}{2\alpha} \right)(1-2\alpha) \left( 1 - 2z \right).
        \label{eq:tDef}
    \end{equation}
    We first show that $t(z) < 1$ as long as $z > 0$ and $p < 1/2$.
    Algebraic manipulations indicate that for $p < 1/2$,
    \begin{equation*}
        \frac{1-p}{2(1-\alpha)} + \frac{p}{2\alpha} + \left( \frac{1-p}{2(1-\alpha)} - \frac{p}{2\alpha} \right)(1-2\alpha) = 1.
    \end{equation*}
    Since $z < 1/2$, we have that $1 - 2z < 1$, and therefore \eqref{eq:tDef} is of the form $a + bc$, where $a, b > 0$, $a + b = 1$, and $c < 1$. We wish to show that $a + bc < 1$. Using the fact that $b = 1-a$, an equivalent statement is
    \begin{equation*}
        a + (1-a) c < 1 \iff (1-a)c < 1-a,
    \end{equation*}
    which holds as long as $c < 1$.

    Next, observe that since $t(z) < 1$,
for any $\theta \in [0,1]$
    \begin{eqnarray*}
        \sum_{n=1}^{\infty} \Pr \left(|\tnhat - \theta| > \Delta \right) &\leq& \sum_{n=1}^{\infty} \frac{1-\Delta}{\Delta} t(z)^{n} \\
        &=& \frac{1-\Delta}{\Delta} \left( \frac{1}{1-t(z)} - 1 \right) < \infty.
    \end{eqnarray*}
    By the Borel-Cantelli lemma, this guarantees that \eqref{eq:asCond} holds. Finally, let $z = \min_{i} \set{z_{i}}_{i=1}^{N}$ and $p$ be the maximum noise level observed throughout the sampling procedure. For $\lambda < 2$, we have $z > 0$. Further, since $p$ is computed by taking the tail of a Gaussian distribution in the direction away from the mean, we have $p < 1/2$. This guarantees convergence of the discretized form of the PFHS algorithm.
\end{proof}

The analysis of probabilistic search algorithms has been a topic of significant study over more than fifty years \cite{horstein1963sequential,burnashev1974interval,pelc2002searching,karp2007noisy,or2008bayesian,waeber2013bisection,tsiligkaridis2016asynchronous}.
While the search fractions used in PFHS are derived from the noiseless setting and therefore suboptimal, deriving an optimal policy for the noisy case is intractable due to the combinatorial explosion of potential posterior distributions. Determining approximate solutions for the noisy case via reinforcement learning is an important topic that lies beyond the
scope of this work.

\subsection{Gaussian Process Level Set Estimation}
\label{sec:gpLSE}

Given a means of handling noisy measurements, we now consider the problem of LSE in the case where the underlying function $f$ is a GP. Formally, a GP is a collection of random variables, one for each value of $f(x)$, for which every finite subset forms a Gaussian random vector \cite{williams2006gaussian}. A GP is characterized by its mean function $\mu(x) = \bE[f(x)]$ and its covariance/kernel function $k(x,x') = \bE[(f(x) - \mu(x))(f(x') - \mu(x'))]$, which governs the smoothness of the
function over the domain, which in our case is $[0,1]^{d}$. In the GP-LSE problem, we assume measurements are corrupted by Gaussian noise, so that $Y_{i} = f(X_{i}) + \zeta_{i}$ with $\zeta_{i} \sim \sN(0, \sigma^{2})$. After obtaining measurements $Y_{1}, \dots, Y_{n}$ at corresponding locations $X_{1}, \dots, X_{n}$, the posterior mean and covariance can be obtained as
\begin{eqnarray*}
    \mu_{n}(x) &=& k_{n}(x)^{T}(K_{n} + \sigma^{2}I)^{-1} y_{n} \\
    k_{n}(x,x') &=& k(x,x') - k_{n}(x)^{T}(K_{n} + \sigma^{2}I)^{-1} k_{n}(x'),
\end{eqnarray*}
where $k_{n}(x) = \begin{bmatrix} k_{n}(x,X_{1}) & k_{n}(x,X_{2}) & \dots & k_{n}(x,X_{n})\end{bmatrix}^{T} \in \bR^{n}$, $K_{n} \in \bR^{n \times n}$ is the positive-definite kernel matrix whose $i,j$th entry is $k_{n}(X_{i}, X_{j})$, and $y_{n} \in \bR^{n}$ is the vector of measurements. The GP model is frequently used in environmental applications \cite{diggle1998model,wong2004comparison,ma2018data} (often referred to as \emph{kriging} in this context \cite{stein2012interpolation}), and the
ability to measure posterior variance has led to a number of approaches to active learning in GPs \cite{gotovos2013active,hitz2014fully,hoang2014nonmyopic,bogunovic2016truncated,inatsu2019active}. To apply our proposed FHS to the GP-LSE problem, we consider a subset of GPs wherein one coordinate of the level set boundary is a function of all other coordinates, as depicted in Fig.~\ref{fig:gpExample}. This assumption is similar to the \emph{boundary fragment} assumption, which has been widely
studied within the nonparametric active learning literature \cite{scott2006minimax,castro2008minimax,korostelev2012minimax}. This assumption reduces the superlevel set to a single, simply-connected region, which commonly holds in environmental
applications \cite{cruz2010adaptive,zhou2013spatial}. In this section, we restrict ourselves to the two-dimensional case for clarity, but the ideas presented can easily be extended to higher dimensions.

\begin{figure}[t]
    \centering
    \includegraphics[width=3.0in]{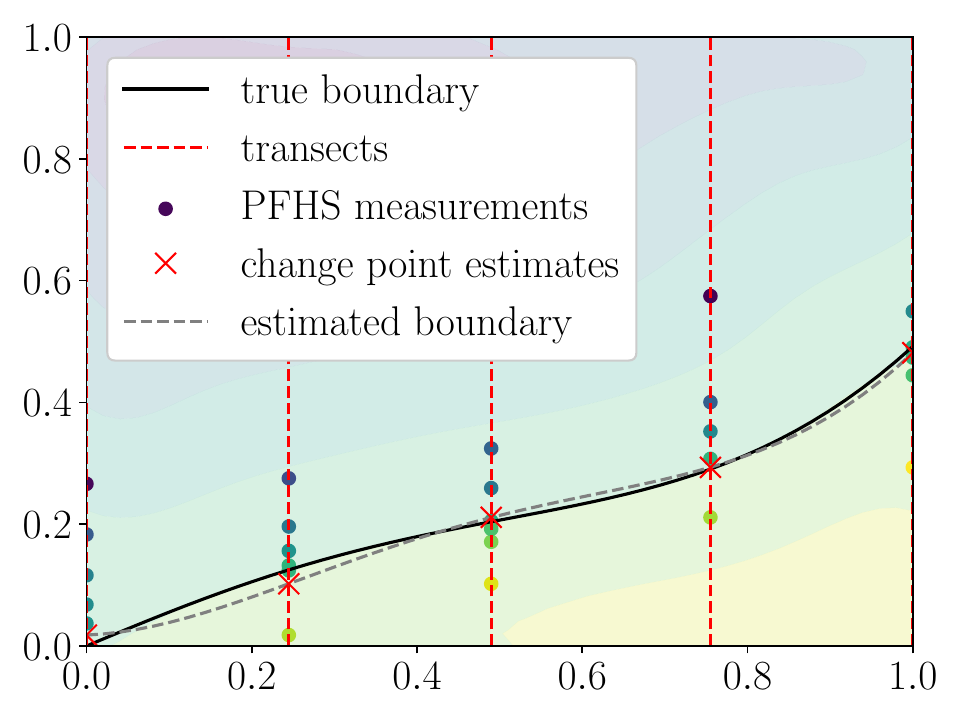}
    \caption{Example Gaussian process level set estimation when the boundary is a function of the first coordinate. The unit interval is split into five equally-spaced transects, and a PFHS procedure is used to localize the change point along each transect. Change point estimates are then used to estimate the boundary via GP regression.}
    \label{fig:gpExample}
\end{figure}

Under the above assumption, our approach to GP-LSE is as follows. Assume without loss of generality that the level set boundary $\partial S$ is a function of the first coordinate. We split the unit interval into a series of transects and perform a one-dimensional PFHS to localize $\partial S$ along each transect, initializing transects as described in Sec.~\ref{sec:seqInit}.
This process is depicted in Fig.~\ref{fig:gpExample}. As a result, our goal is to estimate the one-dimensional function $\partial S$ assuming it is a GP. Along each transect, we run PFHS and treat the estimated change point
as a noisy measurement of $\partial S$ at the transect location. Clearly the accuracy of estimating $\partial S$ is governed by the number of transects and the accuracy of localizing the change point along each transect. In this case, the number of transects
governs the approximation error and the stopping error along each transect defines the estimation error. In our experiments, we tune these parameters though a grid search, and a theoretical characterization of the optimal balance between these parameters is an important topic of future study.

\subsubsection{Initialization and Policy Calculation}
\label{sec:seqInit}

Rather than beginning the search from the origin at each transect, we make use of information from previous transects and initialize the search from the previous change point estimate. 
In the noiseless case, this initial sample reduces the interval size, and we then compute the
optimal policy for the resulting interval length using Alg.~\ref{alg:Ent}. In the noisy case, the initial sample alters the distribution $\pi_{n}$, and we wish to derive an equivalent notion of interval reduction in order to determine the appropriate policy for each transect. Recall that in the noiseless case, the length of the feasible interval corresponds to the exponentiated differential entropy. We therefore compute the \textit{effective interval size} based on the exponentiated
differential entropy after one sample has been obtained. Let $X_{0}$ denote the initial sample location and $p$ denote the corresponding derived probability of error. Following the update equations \eqref{eq:posUpdate} and \eqref{eq:negUpdate}, in the case where $Y_{0} > \gamma$, we have
\begin{align*}
H_{0} &= -\int_{0}^{X_{0}} \frac{p}{X_{0} \circ p} \log\left( \frac{p}{X_{0} \circ p} \right) d\theta - \\
& \qquad \int_{X_{0}}^{1} \frac{1-p}{X_{0} \circ p} \log\left( \frac{1-p}{X_{0} \circ p} \right) d\theta \\
    &= \log\left( X_{0} \circ p \right) - \\
    & \qquad \frac{1}{X_{0} \circ p} \left( p X_{0} \log(p) + (1-p) (1 - X_{0}) \log(1-p) \right).
\end{align*}
Exponentiating gives
\begin{equation*}
    e^{H_{0}} = \left( X_{0} \circ p \right)\left( p^{-pX_{0} / (X_{0} \circ p)} (1-p)^{-(1-p)(1-X_{0}) / (X_{0} \circ p)} \right).
\end{equation*}
Similarly, in the case where $Y_{0} < \gamma$, we have
\begin{equation*}
    e^{H_{0}} = \left( X_{0} * p \right)\left( p^{-p(1-X_{0}) / (X_{0} * p)} (1-p)^{-(1-p)X_{0} / (X_{0} * p)} \right).
\end{equation*}
This generalizes the notion of interval length to the case of noisy measurements, and when $p = 0$ is exactly equal to the resulting interval length.
For each transect, we obtain the initial measurement, compute the corresponding effective interval size (exponentiated entropy), and then compute the policy corresponding to this length via Alg.~\ref{alg:Ent}.

\section{Simulations \& Experiments}
\label{sec:simulations}

\subsection{Performance on One-Dimensional Step Functions}

In this section, we verify the performance of the proposed FHS and PFHS policies. A thorough empirical investigation of the noiseless setting was performed in the conference version of this work \cite{kearns2019optimal}.
Here we demonstrate the reduced cost (as defined by Eq.~\ref{eq:totCost}) using PFHS compared to FHS in the noisy setting. We then benchmark FHS and PFHS against the existing QS and UTB algorithms in a time-penalized search scenario.

\subsubsection{Comparison of FHS and PFHS}

\begin{figure*}[t]
    \centering
    \begin{subfigure}[t]{0.32\textwidth}
        \centering
        \includegraphics[width=\textwidth]{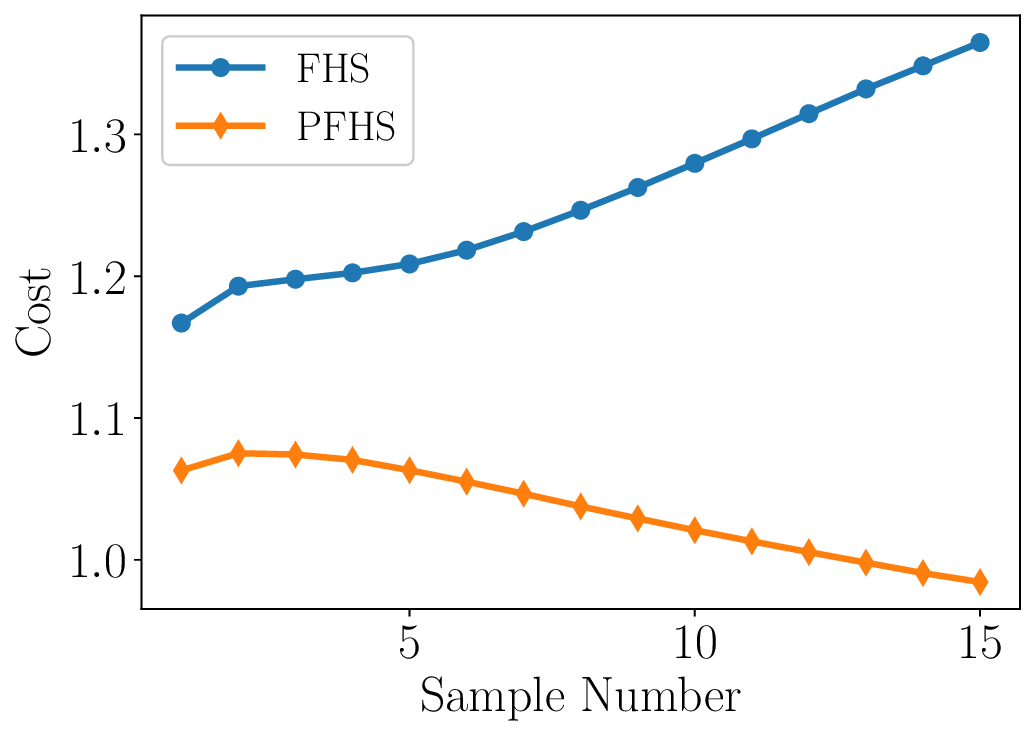}
        \caption{}
    \end{subfigure}
    \hfill
    \begin{subfigure}[t]{0.32\textwidth}
        \centering
        \includegraphics[width=\textwidth]{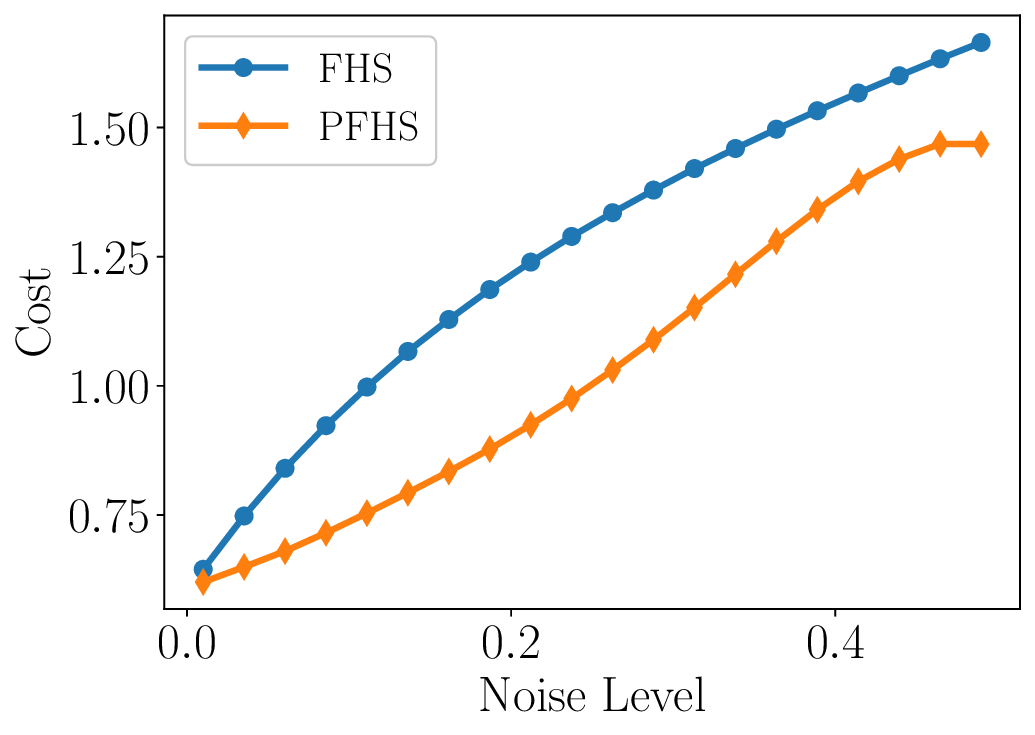}
        \caption{}
    \end{subfigure}
    \hfill
    \begin{subfigure}[t]{0.32\textwidth}
        \centering
        \includegraphics[width=\textwidth]{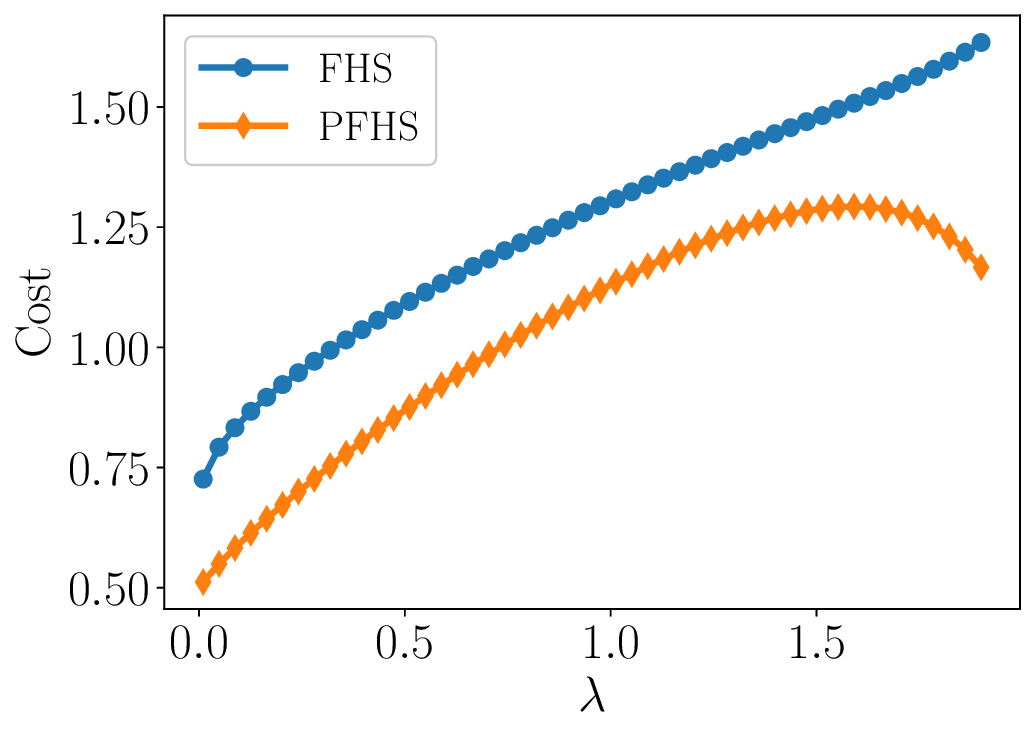}
        \caption{}
    \end{subfigure}
    \caption{Performance of FHS and PFHS in the case of noisy measurements, where cost is defined by \eqref{eq:totCost}. (a) Cost as a function of number of samples. (b) Cost as a function of noise level. (c) Cost as a function of tuning parameter $\lambda$. The performance improvement is most significant when more measurements are taken and for large $\lambda$.}
    \label{fig:fhsVsPfhs}
\end{figure*}

We first compare the performance of FHS and PFHS in the case of binary measurements that are erroneous with constant probability $p$. We report performance over fifteen measurements, considering twenty noise levels between 0.01 and 0.49, and fifty values of $\lambda$ between 0.01 and 1.9. For each configuration, we perform searches over 100
uniformly-spaced values of $\theta$ in the interval $[0,1]$, running 100 Monte Carlo simulations for each value of $\theta$, and report the average cost as defined by \eqref{eq:totCost}. 
In this noisy setting, the interval length for FHS does not reflect the uncertainty in the change point due to the erroneous measurements. By \cite[Thm. 1]{lipor2017distance}, the entropy corresponds to four times the absolute error in the change point. Hence, we use $4 \abs{\hat{\theta}_{N} - \theta}$ when computing \eqref{eq:totCost}.

Fig.~\ref{fig:fhsVsPfhs} shows the cost as a function of sample number (averaged over $p, \lambda$), noise level (averaged over $N, \lambda$), and tuning parameter $\lambda$ (averaged over $p, N$) for each algorithm.
First, we see that the benefits of PFHS are most apparent as more samples are obtained due to the convergent behavior of PFHS. For FHS, the error in estimating $\theta$ can actually increase with more samples, since one erroneous measurement can bias the estimate away from the true value.
Second, while no clear trend in improvement versus noise level is seen, the greatest \textit{percent} improvement is obtained at noise levels below 0.2. This behavior is likely due to the fact that PFHS uses the policy derived from the noiseless case, whose suboptimality is more apparent as the noise level increases. Third, the benefits of PFHS are more apparent as $\lambda$ increases. A closer inspection of entropy and distance reveals that for large $\lambda$, both algorithms travel a
similar distance (making only small movements), but PFHS has a much lower entropy due to its ability to incorporate knowledge of the noise level.
In all cases, PFHS outperforms FHS, with an average cost reduction ranging from 6\% at $p = 0.01$ to 27\% at $p = 0.14$.

\subsubsection{Cost as a Function of Sampling Time}

\begin{figure*}[!t]
    \centering
    \begin{subfigure}[t]{0.32\textwidth}
        \centering
        \includegraphics[width=\textwidth]{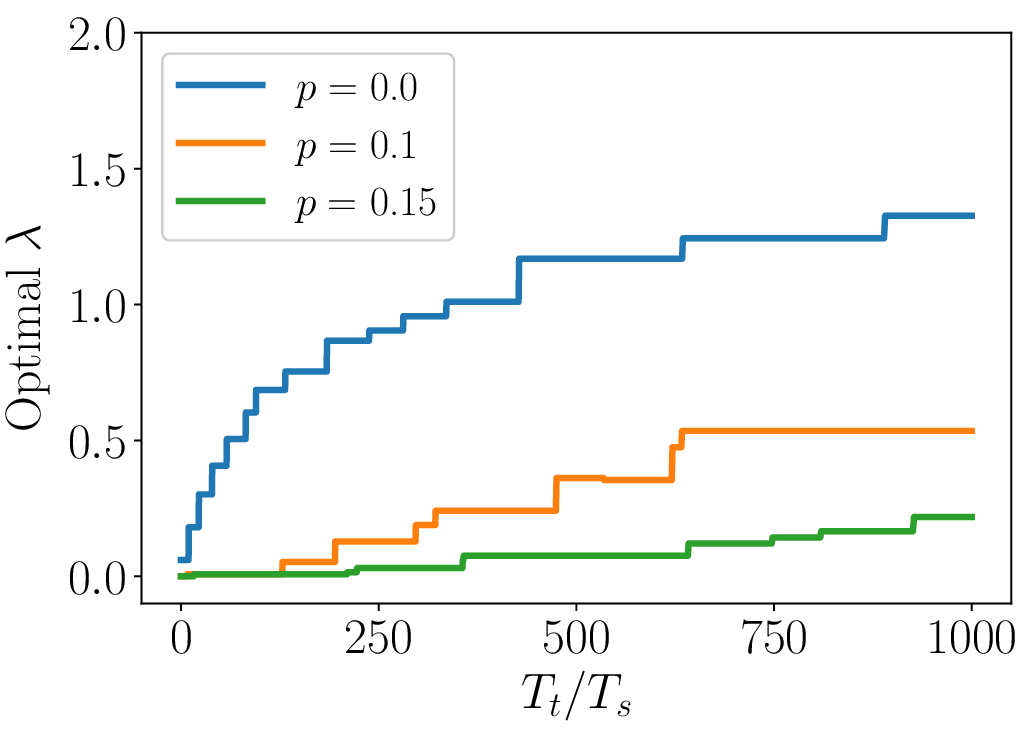}
        \caption{}
    \end{subfigure}
    \hfill
    \begin{subfigure}[t]{0.32\textwidth}
        \centering
        \includegraphics[width=\textwidth]{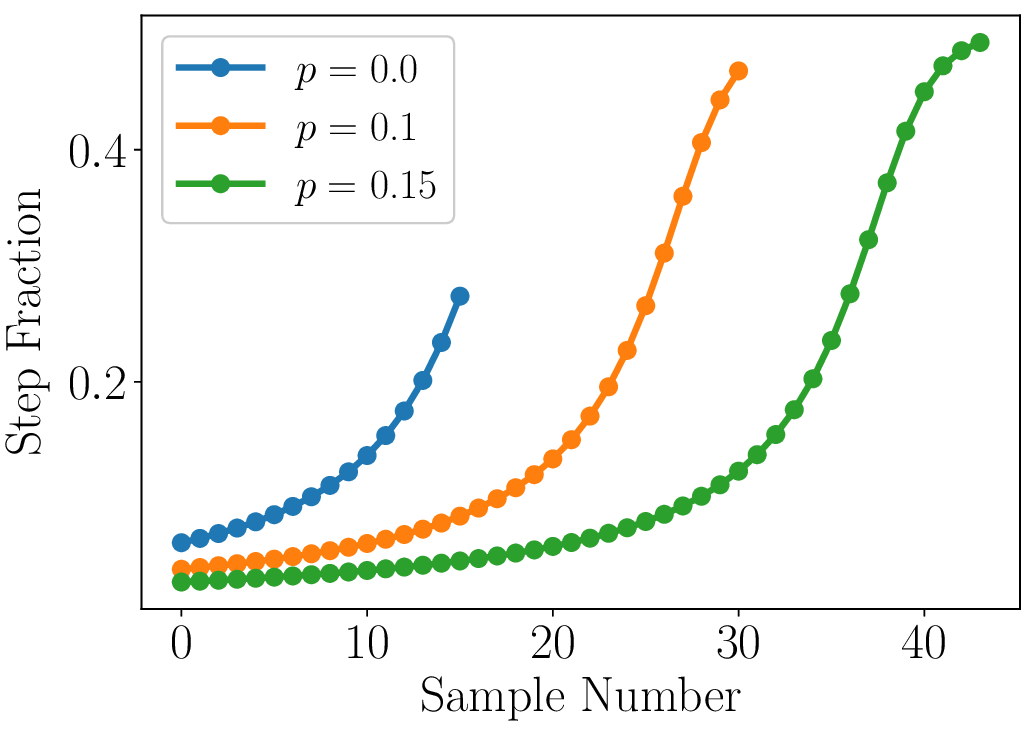}
        \caption{}
    \end{subfigure}
    \hfill
    \begin{subfigure}[t]{0.32\textwidth}
        \centering
        \includegraphics[width=\textwidth]{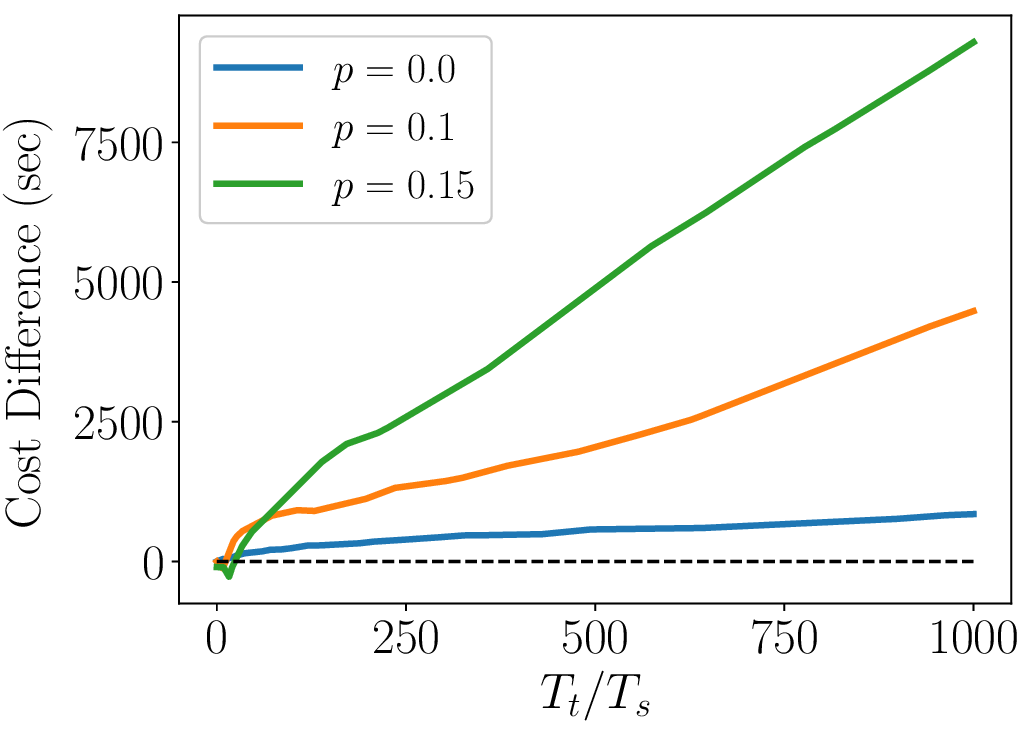}
        \caption{}
    \end{subfigure}
    \caption{Performance of noisy search algorithms for varying noise level $p$ when optimized for total sampling time as defined by \eqref{eq:totCostTime}. (a) Optimal tuning $\lambda$ as a function of the ratio of travel time $T_{t}$ to sampling time $T_{s}$. (b) Policies selected by PFHS for each noise level considered. (c) Cost difference between PQS and proposed PFHS algorithms. As the noise level increases, PFHS favors entropy reduction (smaller $\lambda$) over distance penalization.}
    \label{fig:NCostCompare}
\end{figure*}

To minimize the total time that a vehicle takes to complete a search, we consider a cost function of the form
\begin{equation}
    J_T(z_1, \dots, z_N) = T_s N + T_t D,
    \label{eq:totCostTime}
\end{equation}
where $T_s$ and $T_t$ represent the time per sample and time per unit distance traveled, respectively, and $N$ and $D$ represent the number of samples and total distance. In order to minimize this cost in expectation, we first calculate the number of samples, $N_\lambda$, and total distance, $D_\lambda$, expected for the optimal policy for each value of $\lambda$ to reach a final interval size smaller than desired error $\varepsilon$ using Alg.~\ref{alg:Ent}. We then select the value of $\lambda$ that minimizes the total search time,
\begin{equation}
    \lambda^* = \argmin_\lambda T_s N_\lambda + T_t D_\lambda.
    \label{eq:optLambda}
\end{equation}
For the noisy setting, the expected interval size cannot be computed in closed form. We instead evaluate the sample mean of the interval size, computed over a range of 1,000 values of $\theta \in [0,1]$ and 100 Monte Carlo trials for each value of $\theta$.


We compare the performance of the above method with the existing probabilistic QS (PQS) algorithm for distance-penalized search in one dimension. We consider the same grid of 1,000 values of $\theta$ for 1,000 different ratios of $T_{t}/T_{s}$ in the range of $1 \times 10^{-4}$ to $1 \times 10^{3}$, taking $T_s = 100$ as the base sampling cost.
Fig.~\ref{fig:NCostCompare}(a) shows the value of $\lambda^{*}$ selected by \eqref{eq:optLambda} for noise level $p \in \set{0, 0.1, 0.15}$. As expected, as $T_t$ increases, a higher value of $\lambda^*$ is selected, taking more samples while being less likely to overshoot the change point. Additionally, as the noise level increases, lower values of $\lambda^{*}$ are selected. This results in a search that favors entropy reduction in order to account for the
information loss incurred by noisy measurements.
Fig.~\ref{fig:NCostCompare}(b) shows the PFHS policies selected for each noise level at a ratio $T_{t}/T_{s} = 250$. As expected, more samples are required as the noise level increases, with the lower values of $\lambda$ resulting in a larger maximum step size. The ability of PFHS to utilize small step fractions at early stages allows the algorithm to keep the total distance traveled low while still converging rapidly.
Fig.~\ref{fig:NCostCompare}(c) shows the cost difference between PQS and PFHS for each noise level. In nearly all cases, PFHS outperforms PQS, with a greater difference as both the noise level and the ratio $T_{t}/T_{s}$ increase. Although difficult to see, PQS does outperform PFHS by a small amount for a noise level of $p = 0.15$ and a ratio of $T_{t}/T_{s} \leq 24$. However, PQS obtains a performance improvement of less than 4\%, whereas FHS obtains as much as a 15\% improvement
as the ratio of travel to sample time increases. 


\subsection{Performance on GP-LSE}
\label{sec:gpExperiments}

\begin{figure*}[t]
    \centering
    \begin{subfigure}[t]{0.32\textwidth}
        \centering
        \includegraphics[width=\linewidth]{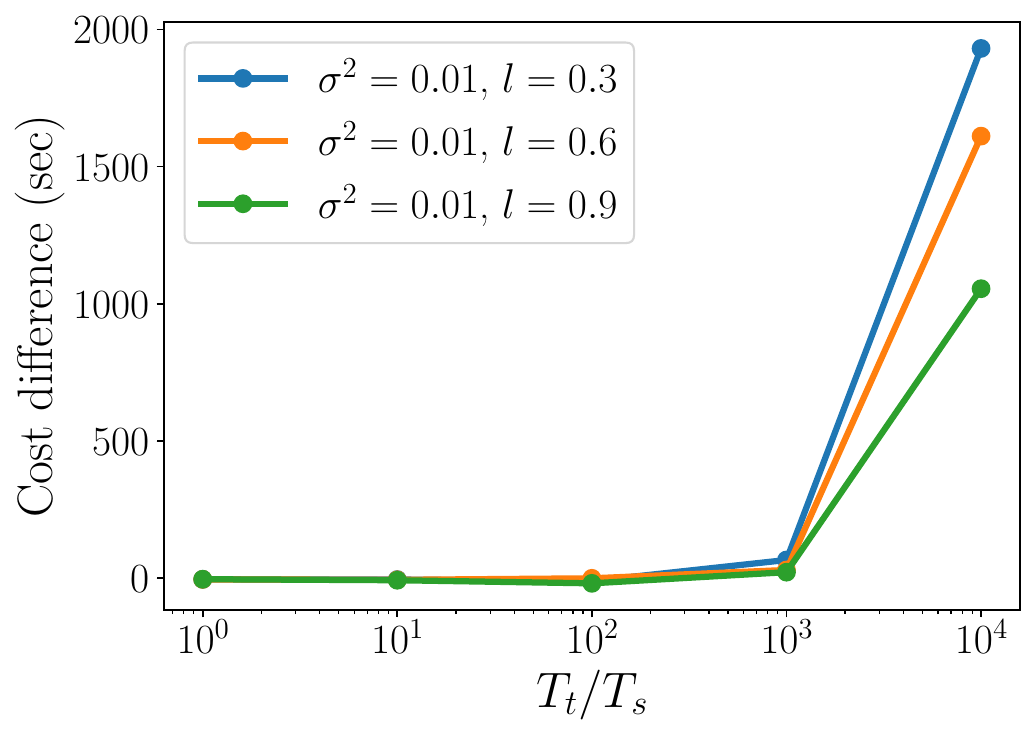}
        \caption{}
    \end{subfigure}
    \hspace{3em}
    \begin{subfigure}[t]{0.32\textwidth}
        \centering
        \includegraphics[width=\linewidth]{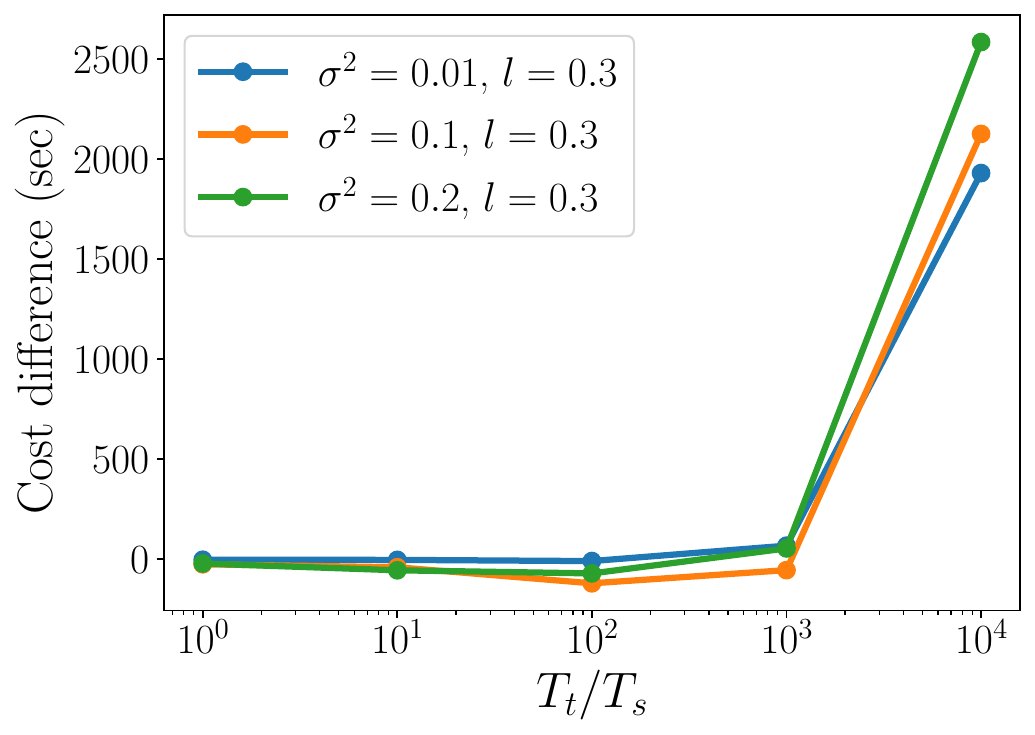}
        \caption{}
    \end{subfigure}
    \caption{Difference in cost between TruVaR and proposed PFHS for GP-LSE on synthetic data. (a) Fixed noise level $\sigma^{2}$ and varying lengthscale $l$. (b) Varying noise level $\sigma^{2}$ and fixed lengthscale $l$. The proposed PFHS obtains the most significant benefit when the ratio of travel time to sample time ($T_{t}/T_{s}$) is large.}
    \label{fig:gpSyntheticResults}
\end{figure*}

Next, we examine the performance of our approach to GP-LSE using the proposed PFHS algorithm. As a benchmark, we compare to the state-of-the-art Truncated Variance Reduction (TruVaR) algorithm \cite{bogunovic2016truncated}, which is designed explicitly for cost-sensitive GP-LSE. To perform LSE, TruVaR maintains estimates of the super-level and sub-level sets, as well as a third set for points whose level set membership is uncertain. Points are placed into the super/sub-level set estimates only after the algorithm is
sufficiently confident they lie above/below the level set threshold, with confidence estimates being obtained from the GP model. TruVaR selects samples based on the ratio of (truncated) variance reduction to cost; as a result, it considers cost myopically. 

We consider two-dimensional GP fields with boundaries satisfying the assumption described in Section~\ref{sec:gpLSE}. 
In all simulations, we provide TruVaR with the true kernel parameters used to generate the two-dimensional GP field under consideration. We set the parameter $\beta_{(i)} = a \log\left( M t_{(i)}^{2} \right)$ as in \cite{bogunovic2016truncated}, where $t_{(i)}$ denotes the time at which the epoch starts, $M$ denotes the number of elements in the two-dimensional field, and $a$ is a constant. We found the best performance resulted from setting $a =
0.0001$. All other parameters were set according to the recommendations in \cite{bogunovic2016truncated}. We measure the error in terms of the symmetric difference between the true and estimated super-level set divided by the total number of points in the field, i.e.,
\begin{equation}
    E = \frac{1}{M} \abs{S \triangle \hat{S}},
    \label{eq:lsError}
\end{equation}
where $M$ is the total number of points in the field and $\hat{S}$ denotes the estimated super-level set. When considering the LSE problem as binary classification, the above may be viewed as the classification error.

\begin{table*}[ht!]
    \centering
    \begin{tabular}{ | c | c || c || c || c || c | }
        \hline
        \multirow{2}{*}{} & Sampling Time (s) & 8 & 8 & 30 & 30 \\
        & Velocity (km/hr) & 32 & 65 & 32 & 65 \\
        \hline
        \multirow{2}{*}{PFHS} & Search Cost (hr) & 9.58 & 4.90 & 10.46 & 5.54 \\
        & Error (\%) & 2.87 & 3.26 & 3.16 & 3.37 \\
        \hline
        \multirow{2}{*}{TruVaR} & Search Cost (hr) & 25.1 & 12.75 & 24.88 & 12.96 \\
        & Error (\%) & 14.62 & 14.56 & 14.30 & 14.86 \\
        \hline
    \end{tabular}
    \caption{Total sampling time (in seconds) and estimation error in air quality data following the Camp Fire in November 2018. Region considered and example sampling pattern of PFHS are depicted in Fig.~\ref{fig:fireMap}. For all sampling times and velocities considered, PFHS achieves a significant reduction in both cost and estimation error.}
    \label{tab:fire}
\end{table*}

\subsubsection{Synthetic GP Data}
\label{sec:syntheticGP}

We first compare algorithm performance on synthetic GP data. To ensure the boundary assumption described in Section~\ref{sec:gpLSE} holds, we begin by generating a one-dimensional GP defining the level set boundary. 
We then generate a two-dimensional GP field by taking the true value to be the positive or negative distance from the boundary, obtaining 500 measurements corrupted by Gaussian noise with variance 0.0001.
Finally, we fit a two-dimensional GP to the obtained measurements. An example boundary and
two-dimensional field are shown in Fig.~\ref{fig:gpExample}. For both the one-dimensional boundary and the two-dimensional field, we use the radial basis function (RBF) kernel. For the boundary, we consider lengthscales of 0.3, 0.6, and 0.9 to simulate varying degrees of smoothness in $\partial S$. The two-dimensional field is always fit using a lengthscale of 0.1. Both PFHS and TruVaR are given the true kernel parameters when performing GP-LSE. In all experiments, we generate a field of
size $21 \times 20$; this size is chosen largely due to the high computation time required by TruVaR. To test our approach to handling noisy
measurements, we consider noise variances $\sigma^{2} \in \set{0.01, 0.1, 0.2}$.

To select the number of transects and stopping
error for PFHS, we generate 100 examples of GP fields and boundaries using the above procedure and perform a grid search over both parameters. We then select the parameters that give the lowest total cost while maintaining an average error below 8\% for noise variances $\sigma^{2} \in \set{0.01, 0.1}$ and an error below 11\% for $\sigma^{2} = 0.2$.\footnote{The latter choice was made to match the accuracy achieved by TruVaR for the given parameters.} This procedure is used to select the best
parameters for each lengthscale and noise level under consideration. Note that an equivalent procedure would be required to select the GP kernel and bandwidth parameters for TruVaR; however, to avoid the large computational cost of tuning these parameters, we provide TruVaR with the true kernel used to generate the GP fields. Finally, we compare both algorithms on 100 separate GP fields for each lengthscale.

Fig.~\ref{fig:gpSyntheticResults} displays the average cost difference between TruVaR and PFHS over the 100 random fields, showing the cost difference as a function of the ratio $T_{t}/T_{s}$ for varying values of (a) lengthscale and (b) noise variance. In both cases, we see that for high ratios of travel-to-sample time, FHS outperforms TruVaR by a significant margin. Fig.~\ref{fig:gpSyntheticResults}(a) shows that this improvement reduces with lengthscale, indicating that PFHS excels when the
boundary is least smooth. Fig.~\ref{fig:gpSyntheticResults}(b) shows the improvement for different noise levels and indicates that PFHS obtains the largest improvement for higher noise levels. 
Although difficult to see from the figures, TruVaR does outperform FHS for $T_{t}/T_{s} \in \set{1, 10, 100}$. However, the mean and maximum improvement are 58 sec and 269 sec, respectively, whereas PFHS achieves a mean/maximum improvement of 1200/3900 sec. Further, PFHS typically obtains an error that is 1-4\% lower than that of TruVaR for these values of $T_{t}/T_{s}$, indicating a more careful tuning of parameters may allow PFHS to obtain better performance. Hence, while FHS is most beneficial when travel time is significant relative to measurement time, it is still competitive even for low values of $T_{t}$. Further, we see that by treating the cost of travel nonmyopically, significant performance benefits can be obtained, even though FHS is not explicitly designed for sampling GPs.

Finally, we comment that one additional drawback of TruVaR is that of computational complexity. To perform sample selection, TruVaR must compute the posterior variance after sampling for every location in the set of uncertain points. As a result, the average computation time for each search in the above experiments was 2.45 sec for TruVaR compared to 0.34 sec for PFHS. While both times are sufficiently small for practical applications, we remark that we chose a small field
size ($21 \times 20$) that would result in limited resolution over large spatial regions. This consideration is especially important when attempting to deploy adaptive sampling algorithms on low-cost mobile sensing devices, which will have limited resources for computation and power.

\subsubsection{Air Quality Data}
\label{sec:airQuality}

Finally, we compare the performance of PFHS, FHS, and TruVaR on real air quality data obtained from the AirNow database \cite{epa2020air}. One potential application of LSE approaches is to provide a high-quality estimate of regions containing high levels of particulate matter. Of particular importance is the problem of rapidly estimating such regions during major wildfire events, such as the 2018 Camp Fire in California \cite{census2020camp} or the more recent series of wildfires
impacting the western U.S. in 2020, which resulted in the worst air quality in the world for major cities such as Portland, OR and San Francisco, CA \cite{oregon2020portland}. 

We consider PM 2.5 data from November 18, 2018, using 124 sensors in the region of Butte County, CA. Since the measurements are spatially sparse, we interpolate the values using two-dimensional GP regression with a summation of RBF and bias kernels optimized and implemented via the GPy package \cite{gpy2014}. We set the threshold at 100 $\mu$g / m$^{3}$, which corresponds to the ``unhealthy for sensitive groups'' level according to the AirNow standards \cite{epa2020air}. We perform LSE
over the region depicted in Fig.~\ref{fig:fireMap}(b), which is approximately 111 km per side.
We consider sampling times of 8 and 30 seconds, corresponding to the extremes of the settling time of the Sensirion SPS30 particulate matter sensor \cite{sensirion2021sps}. This sensor has a precision of $\pm$10 $\mu$g/m$^{3}$; treating errors uniformly throughout this range, we set the noise variance of the GP to that of a uniform distribution with support $[-10,10]$, resulting in
$\sigma^{2} = 20^{2} / 12$. This choice minimizes the KL divergence between the uniform and normal distributions. Finally, we consider travel times of 32 km/hr and 65 km/hr based on the maximum speed of the DJI Matrice 600 UAV.
We provide TruVaR with the two-dimensional kernel used to perform GP regression over the field and set the parameter $a = 6$, as we found that the recommended parameter $a = 1$ resulted in very high estimation errors for the high noise variance considered. 
For PFHS, we model the boundary
using a one-dimensional GP with RBF kernel having lengthscale and variance both set to unity. We search over five transects and set the stop error for each transect to 0.03.

\begin{figure}[t]
    \centering
    \includegraphics[width=1.0\linewidth]{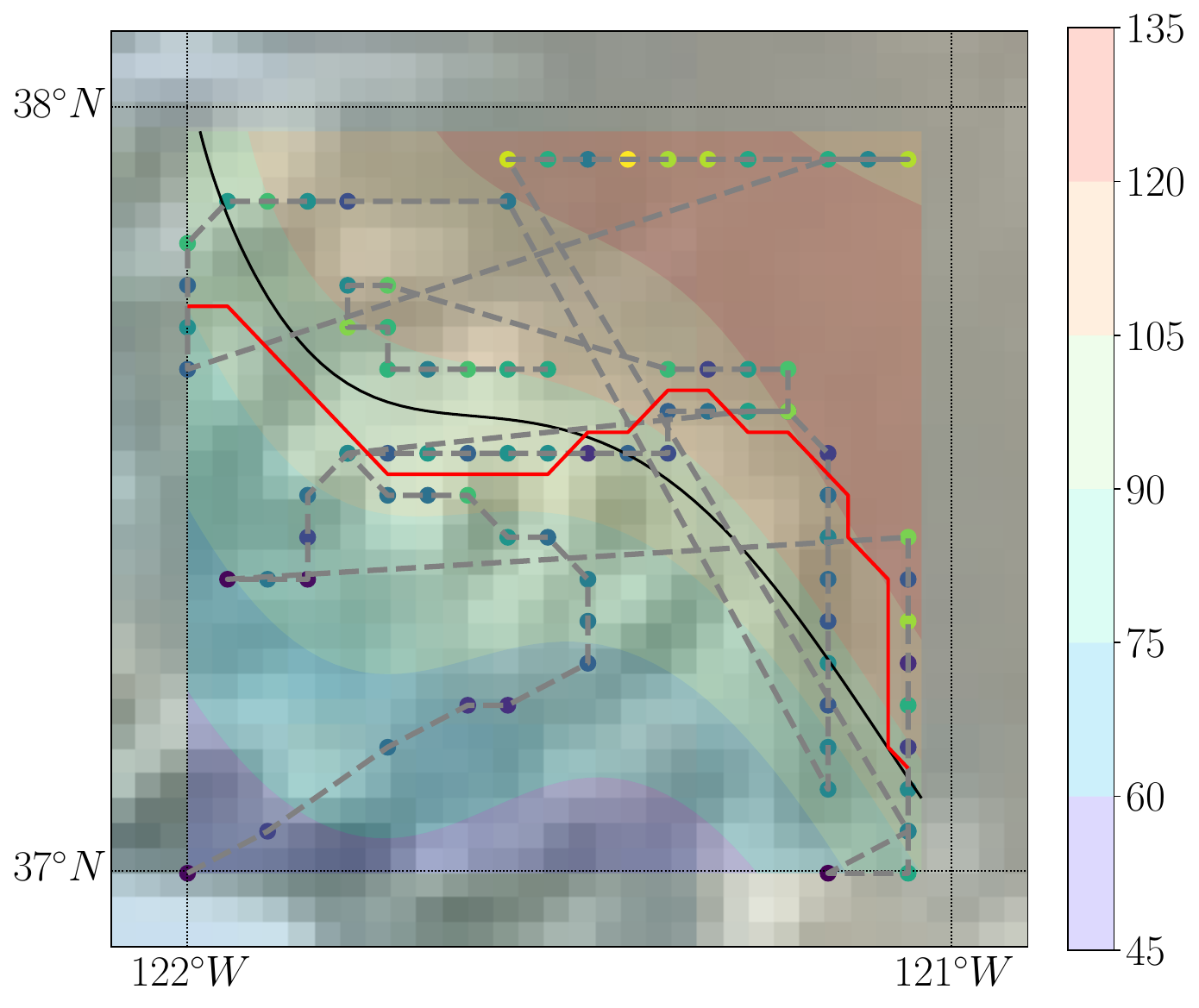}
    \vspace{-1em}
    \caption{Map of sample locations and trajectory followed by TruVaR algorithm on Camp Fire data. Dots denote sample locations, red solid line is estimated boundary, and gray dashed line is path traversed by sensor. Compared to PFHS (see Fig.~\ref{fig:fireMap}(b)), TruVaR does not sufficiently penalize distance in this noisy regime.}
    \label{fig:fireMapTruVaR}
\end{figure}

Table \ref{tab:fire} shows the resulting search cost and error for PFHS and TruVaR in the four scenarios considered. Compared with the results on synthetic data, we see that the performance of TruVaR degrades significantly in this high-noise regime. In all cases, PFHS achieves a lower cost and lower estimation error, typically yielding an estimation error approximately one fifth that of TruVaR at less than half the cost. 
We display the sample locations and path traveled by TruVaR in Fig.~\ref{fig:fireMapTruVaR}. In this high-noise regime, TruVaR places too much emphasis on variance reduction and fails to appropriately penalize for distance traveled. Although not pictured, we also tested the
lower-noise regime and found TruVaR to be more competitive in this setting, focusing samples near the level set boundary and traveling a smaller distance. Further, we note that PFHS relies on the assumption that the superlevel set is a single, connected region with a boundary that can be written as a function of one coordinate. While this assumption is realistic in the case of tracking a wildfire front, it may not be appropriate in other settings (e.g., that considered in \cite{bogunovic2016truncated}), and TruVaR has the added flexibility of discovering superlevel sets consisting of multiple
disjoint regions. Hence, while PFHS provides significant improvements over TruVaR in this experiment, the choice of algorithm is ultimately dependent on the function being sensed and should be informed by expert/domain knowledge.

\section{Conclusions \& Future Work}
\label{sec:conclusion}

We have presented a finite-horizon approach to sensing the change point of a one-dimensional step function that optimally balances the distance traveled and number of samples acquired. We have shown that the resulting policy can be obtained in closed form, making it easily deployable on mobile sensors such as those mounted on a UAV. Aside from outperforming heuristic methods for one-dimensional search, our proposed FHS algorithm outperforms existing methods on the problem of Gaussian
process level set estimation under certain assumptions on the level set boundary.

Our approach to two-dimensional sampling requires localizing the change point over a series of transects. While we optimized both the number of transects and the error per transect numerically, an important open problem is determining the optimal values of these quantities analytically. Another important next step is to incorporate other realistic vehicle costs, such as acceleration and battery life, into the policy calculation.

\appendix[Proofs of Technical Results]

\begin{lemma}
    \label{lem:expLength}
    Let $H_{N}$ be the length of the feasible interval after $N$ measurements. Under the conditions of Theorem~\ref{thm:closedForm}, we have
    \begin{equation}
        \label{eq:expLength}
        \bE \left[ e^{H_{N}} \right] = \prod_{i = 1}^{N} \left( z_{i}^{2} + (1 - z_{i})^{2} \right).
    \end{equation}
\end{lemma}

\begin{proof} 
    First note that under the uniform distribution on the unit interval, the exponentiated differential entropy is the length of the feasible interval after $N$ samples. The proof will proceed by induction on $N$. Consider the base case, $N = 1$, for which it is trivial to show that
    \begin{align*}
        \bE\left[ e^{H_1} \right] = z_{1}^{2} + (1 - z_{1})^{2} = \xi_1 .
    \end{align*}
            
    Now assume that \eqref{eq:expLength} holds for some $N \in \bN$. Sampling some fraction $z_{N+1}$ into the remaining feasible interval $e^{H_{N}}$ results in two potential entropies
    
    \begin{equation*}
    e^{H_{N+1}} = \begin{cases}
        z_{N+1} e^{H_N}, & \text{w.p.} \quad z_{N+1} \\
        (1-z_{N+1})e^{H_n}, & \text{w.p.} \quad 1-z_{N+1}.
        \end{cases}
    \end{equation*}
    
    
    Therefore
    \begin{eqnarray*}
       \bE [e^{H_{N+1}}] & = & z_{N+1}^2 \bE[e^{H_N}]+(1-z_{N+1})^2\bE[e^{H_N}]\\
       & = & \left(z_{N+1}^2+(1-z_{N+1})^2\right) \bE[e^{H_N}] \\
       & = & \prod_{i = 1}^{N+1} \left( z_{i}^{2} + (1 - z_{i})^{2} \right).
    \end{eqnarray*}
\end{proof}

\begin{proof}[Proof of Lemma~\ref{lem:dp}]
    Define $z_{1:n}=z_1,\ldots,z_n$. Using the hypothesis of Lemma~\ref{lem:dp}, 
For any $z_{1:n}$,  
    \begin{equation}
        J(z_{1:n}) \geq J(z_{1:n-1},z_n^*) \geq J(z_{1:n-2},z_{n-1:n}^*) \geq \ldots \geq J(z_{1:n}^*).
    \end{equation}
\end{proof}

\bibliographystyle{IEEEtran}
\bibliography{IEEEabrv,./bibliography}

\end{document}